\documentclass{article}

\usepackage{microtype}
\usepackage{graphicx}
\usepackage{caption}
\usepackage{subcaption}
\usepackage{booktabs} 

\usepackage{amsmath}
\usepackage{amsfonts}
\allowdisplaybreaks
\usepackage{amsthm}
\usepackage{comment}
\usepackage[switch]{lineno}
\usepackage{cancel}
\usepackage{chngcntr}

\newtheorem{theorem}{Theorem}
\newtheorem{assumption}{Assumption}

\newtheorem{lemma}{Lemma}

\newcommand{\p}{\mathcal{P}}
\newcommand{\s}{\mathcal{S}}
\newcommand{\m}{\mathcal{M}}
\newcommand{\act}{\mathcal{A}}

\newcommand{\pol}{\pi}
\newcommand{\dis}{d_\pol}

\newcommand{\vp}{v_\pol}
\newcommand{\rp}{r_\pol}
\newcommand{\pp}{\p_\pol}
\newcommand{\E}{\mathbb{E}}
\newcommand{\Edis}{\E_{\dis}}
\newcommand{\Ep}{\E_\pi}
\newcommand{\fmat}{\Phi}
\newcommand{\fvec}{{\bf \phi}}
\newcommand{\R}{\mathbb{R}}
\newcommand{\prob}{\mathbb{P}}

\newcommand{\param}{\beta}
\newcommand{\proj}{\Pi}
\newcommand{\T}{\mathcal{T}}
\newcommand{\Tp}{\T^{\param}}

\newcommand{\w}{{\bf w}}
\newcommand{\e}{{\bf e}}
\newcommand{\wpi}{\w_{\pol}}

\newcommand{\keyA}{\mathbf{A}}
\newcommand{\keyb}{\mathbf{b}}
\newcommand{\keyE}{\mathbf{E}}
\newcommand{\pbeta}{\pp^{\param}}
\newcommand{\norm}[1]{\left\lVert#1\right\rVert}

\usepackage{todonotes}

\newcommand\numberthis{\addtocounter{equation}{1}\tag{\theequation}}

\usepackage{hyperref}

 \usepackage[accepted]{icml2021}

\icmltitlerunning{Preferential Temporal Difference Learning}

\begin{document}

\twocolumn[
\icmltitle{Preferential Temporal Difference Learning}



\icmlsetsymbol{equal}{*}

\begin{icmlauthorlist}
\icmlauthor{Nishanth Anand}{to,ed}
\icmlauthor{Doina Precup}{to,ed,goo}
\end{icmlauthorlist}

\icmlaffiliation{to}{Mila (Quebec Artificial Intelligence Institute), Montreal, Canada}
\icmlaffiliation{goo}{Deepmind, Montreal, Canada}
\icmlaffiliation{ed}{School of Computer Science, McGill University, Montreal, Canada}

\icmlcorrespondingauthor{Nishanth Anand}{nishanth.anand@mail.mcgill.ca}

\icmlkeywords{Reinforcement Learning, RL, Temporal Difference Learning, PTD, Preferential TD, Function approximation, Partial observability}

\vskip 0.3in
]



\printAffiliationsAndNotice{}  

\begin{abstract}
Temporal-Difference (TD) learning is a general and very useful tool for estimating the value function of a given policy, which in turn is required to find good policies. Generally speaking, TD learning updates states whenever they are visited. When the agent lands in a state, its value can be used to compute the TD-error, which is then propagated to other states. However, it may be interesting, when computing updates, to take into account other information than whether a state is visited or not. For example, some states might be more important than others (such as states which are frequently seen in a successful trajectory). Or, some states might have unreliable value estimates (for example, due to partial observability or lack of data), making their values less desirable as targets. We propose an approach to re-weighting states used in TD updates, both when they are the input and when they provide the target for the update. We prove that our approach converges with linear function approximation and illustrate its desirable empirical behaviour compared to other TD-style methods.
\end{abstract}

\section{Introduction}
\label{section:introduction}

The value function is a crucial quantity in reinforcement learning (RL), summarizing the expected long-term return from a state or a state-action pair. The agent uses this knowledge to make informed action decisions. Temporal Difference (TD) learning methods~\citep{sutton1988learning} enable updating the value function before the end of an agent's trajectory by contrasting its return predictions over consecutive time steps, i.e., computing the {\em temporal difference error} (TD-error). State-of-the-art RL algorithms, e.g.~\citet{mnih2015human, schulman2017proximal} use this idea coupled with function approximation. 

TD-learning can be viewed as a way to approximate dynamic programming algorithms in Markovian environments~\citep{barnard1993temporal}. But, if the  Markovian assumption does not hold (as is the case when function approximation is used to estimate the value function), its use can create problems~\citep{gordon1996chattering,sutton2018reinforcement}. To see this, consider the situation depicted in Figure \ref{fig:linear_mdp_a}, where an agent starts in a fully observable state and chooses one of two available actions. Each action leads to a different long-term outcome, but the agent navigates through aliased states that are distinct but have the same representation before observing the outcome. This setting poses two challenges:

\textbf{1. Temporal credit assignment}: The starting state and the outcome state are temporally distant. Therefore, an efficient mechanism is required to propagate the credit (or blame) between them. 

\textbf{2. Partial observability}: With function approximation, updating one state affects the value prediction at other states. If the generalization is poor, TD-updates at partially observable states can introduce errors, which propagate to estimates at fully observable states.

TD($\lambda$) is a well known class of \textit{span independent algorithm}~\citep{van2015learning} for temporal credit assignment introduced by~\citet{sutton1988learning} and further developed in many subsequent works, e.g.~\citet{singh1996reinforcement, seijen2014true}, which uses a \textit{recency} heuristic: any TD-error is attributed to preceding states with an exponentially decaying weight. However, recency can lead to inefficient propagation of credit~\citep{aberdeen2004filtered, harutyunyan2019hindsight}. Generalized TD($\lambda$) with a state-dependent $\lambda$ can tackle this problem~\citep{sutton1995td, sutton1999between, sutton2016emphatic, yu2018generalized}. This approach mainly modifies the target used in the update, not the extent to which any state is updated. Besides, the update sequence may not converge. Emphatic TD~\citep{sutton2016emphatic} uses an independent interest function, not connected to the eligibility parameter, in order to modulate how much states on a trajectory are updated.

In this paper, we introduce and analyze a new algorithm: \textit{Preferential Temporal Difference (Preferential TD or PTD) learning}, which uses a single state-dependent \textit{preference function} to emphasize the importance of a state both as an input for the update, as well as a participant in the target. If a state has a low preference, it will not be updated much, and its value will also not be used much as a bootstrapping target. This mechanism allows, for example, \textit{skipping} states that are partially observable and propagating the information instead between fully observable states. Older POMDP literature~\citep{loch1998using, theocharous2004approximate} has demonstrated the utility of similar ideas empirically. We provide a convergence proof for the expected updates of Preferential TD in the linear function approximation setting and illustrate its behaviour in partially observable environments.

\begin{figure}[H]
    \centering
    \begin{subfigure}{0.9\columnwidth}
        \centering
        \includegraphics[width=0.8\columnwidth, height=3cm]{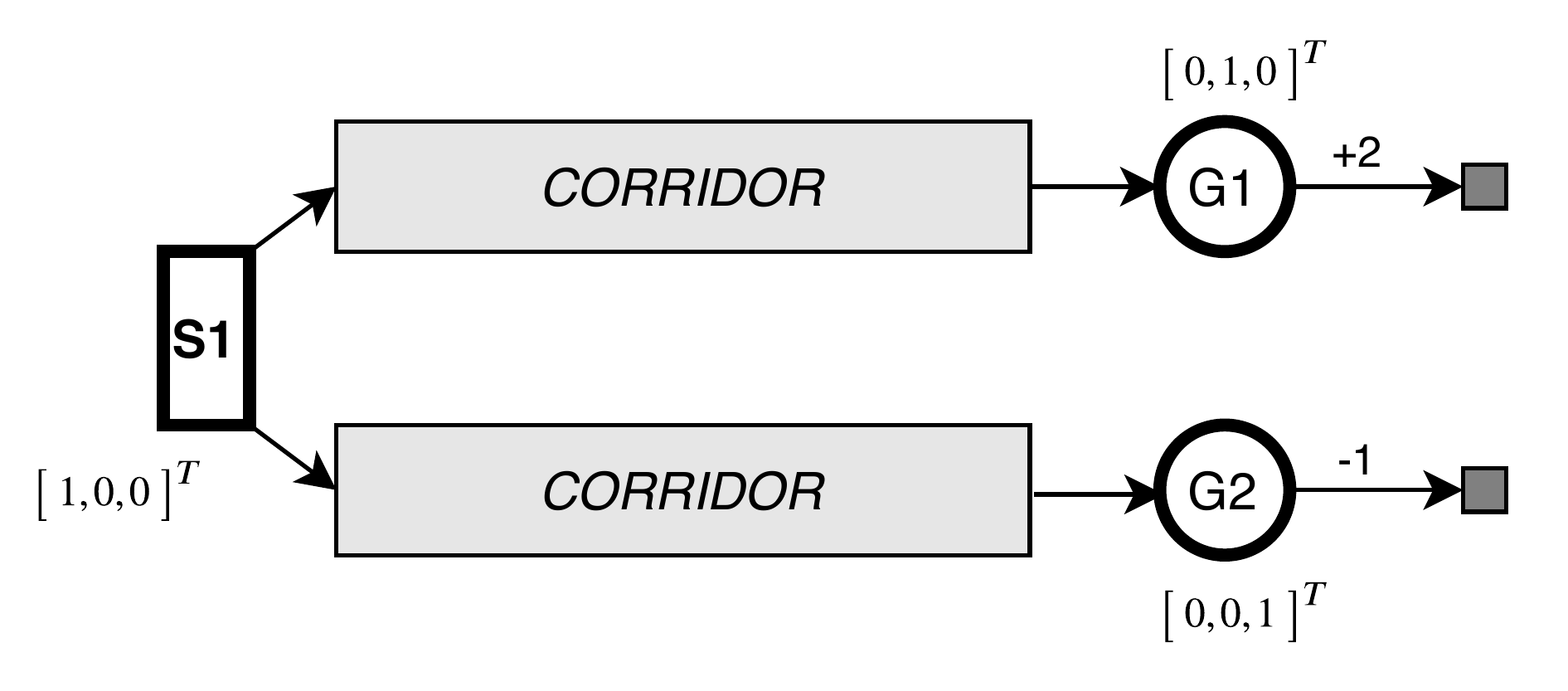}
        \caption{Task 1}
        \label{fig:linear_mdp_a}
    \end{subfigure}
    \vfill
    \begin{subfigure}{\columnwidth}
        \centering
        \includegraphics[width=\columnwidth, height=3cm]{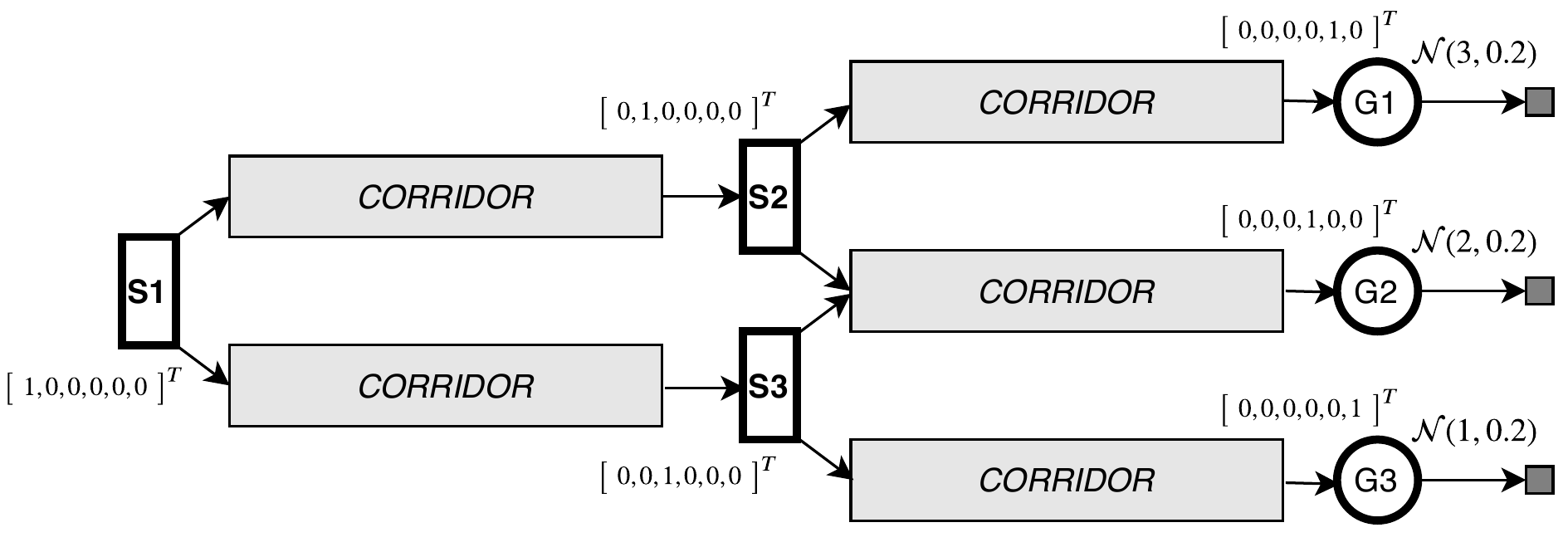}
        \caption{Task 2}
        \label{fig:linear_mdp_b}
    \end{subfigure}
    \caption{Delayed effect MDPs: decision states are shown as boxes, goal states are shown in circles. Feature vectors are specified next to the states. The corridor represents a chain of partially observable states.}
    \label{fig:linear_mdp}
\end{figure}

\section{Preliminaries and Notation}
\label{section:preliminaries}

A Markov Decision Process (MDP) is defined by a tuple $\m = (\s, \act, \p, r, \gamma)$, where $\s$ is the finite set of states, $\act$ is the finite set of actions, $\p(s'|s,a) = \prob\{s_{t+1}=s'|s_t=s, a_t=a\}$ is the transition model, $r:\s \times \act \to \mathbb{R}$ is the reward function, and $\gamma \in [0,1)$ is the discount factor. The agent interacts with the environment in state $s_t \in \s$ by selecting an action $a_t \in \act$ according to its policy, $\pol(a|s)=\prob\{a_t=a|s_t=s\}$. As a consequence, the agent transitions to a new state $s_{t+1}\sim \p(\cdot|s_t,a_t)$ and receives reward $r_{t+1}=r(s_t,a_t)$. We consider the policy evaluation setting, where the agent's goal is to estimate the value function:
\begin{equation}
\label{eq:value_function_definition}
    \vp(s) = \E_{\pol}[G_t | s_t = s],
\end{equation}
where $G_t = \sum_{i=t}^{\infty}\gamma^{i-t} r_{i+1}$ is the discounted return obtained by following the policy $\pol$ from state $s$. If $|\s|$ is very big, $\vp$ must be approximated by using a function approximator. We consider linear approximations:
\begin{equation}
\label{eq:state_value_linearFA}
    v_{\w}(s) = \w^T \fvec(s),
\end{equation}
where $\w \in \mathbb{R}^k$ is the parameter vector and $\fvec(s)$ is the feature vector for state $s$. Note that the linear case encompasses both the tabular case as well as the use of fixed non-linear features, as detailed in~\citet{sutton2018reinforcement}. By defining a matrix $\fmat$ whose rows are $\fvec^T(i),\ \forall i \in \s$, this approximation can be written as a vector in ${\mathbb R}^{|\s|}$ as ${\bf v}_{\w} = \fmat \w$.

TD($\lambda$)~\citep{sutton1984temporal, sutton1988learning} is an online algorithm for updating $\w$, which performs the following update after every time step: 
\begin{align}
    \label{eq:td_etrace}
    \e_t &= \gamma \lambda \e_{t-1} +  \fvec(s_t)\\
    \w_{t+1} &= \w_t + \alpha_t (r_{t+1} + \gamma \w^T \fvec(s_{t+1}) - \w^T\fvec(s_t)) \e_t, \nonumber
\end{align}
where $\alpha_t$ is the learning rate parameter, $\lambda \in [0,1]$ is the eligibility trace parameter, used to propagate TD-errors with exponential decay to states that are further back in time, and $\e_t$ is the eligibility trace.

\section{Preferential Temporal Difference Learning}
\label{section:PTD}

Let $\param:\s \to [0,1]$ be a preference function, which assigns a certain importance to each state. A preference of $\param(s)=0$ means that the value function will not be updated at all when $s$ is visited, while $\param(s)=1$ means $s$ will receive a full update. Note, however, that if $s$ is not updated, its value will be completely inaccurate, and hence it should not be used as a target for predecessor states because it would lead to a biased update. To prevent this, we can modify the return $G_t$ to a form that is similar to $\lambda$-returns~\citep{watkins1989learning, sutton2018reinforcement}, by bootstrapping according to the preference:
\begin{align}
    \label{eq:preferential_td_returns}
    G_t^{\param} &= r_{t+1} + \nonumber \\
    &  \gamma [\param(s_{t+1}) \w^T \fvec(s_{t+1}) + (1 - \param(s_{t+1})) G_{t+1}^{\param}].
\end{align}
The update corresponding to this return leads to the offline Preferential TD algorithm:
\begin{align}
    \label{eq:offline_PTD}
    \w_{t+1} &= \w_t + \alpha \param(s_t) (G_t^{\param} - \w^T\fvec(s_t)) \fvec(s_t), \nonumber \\
    &= \w_t + \\
    &\!\!\! \!\!\!\!\!\!\alpha_t \Big(\underbrace{\param(s_t) G_t^{\param} + (1 - \param(s_t)) \w^T\fvec(s_t)}_{target} - \w^T\fvec(s_t) \Big) \fvec(s_t). \nonumber
\end{align}
The expected target of offline PTD (cf. equation \ref{eq:offline_PTD}) can be written as a Bellman operator:
\begin{theorem}
The expected target in the forward view can be summarized using the operator $$\Tp {\bf v} = B (I - \gamma \pp (I-B))^{-1} (\rp + \gamma \pp B {\bf v}) + (I-B){\bf v},$$
where $B$ is the $|\s| \times |\s|$ diagonal matrix with $\param(s)$ on its diagonal and $\rp$ and $\pp$ are the state reward vector and state-to-state transition matrix for policy $\pi$.
\end{theorem}
We obtain the desired result by considering expected updates in vector form. The complete proof is provided in Appendix \ref{app:operator_proof}.

Using equation ~\ref{eq:preferential_td_returns}, one would need to wait until the end of the episode to compute $G_t^{\param}$. We can turn this into an online update rule using the eligibility trace mechanism~\citep{sutton1984temporal, sutton1988learning}:
\begin{align}
    \label{eq:preferential_etrace}
    \e_t &= \gamma (1- \param(s_t)) \e_{t-1} + \param(s_t) \fvec(s_t), \\
    \w_{t+1} &= \w_t + \alpha_t (r_{t+1} + \gamma \w^T \fvec(s_{t+1}) - \w^T\fvec(s_t)) \e_t, \nonumber
\end{align}
where $\e_t$ is the eligibility trace as before. The equivalence between the offline and online algorithm can be obtained following~\citet{sutton1988learning, sutton2018reinforcement}. 
Theorem~\ref{remark:forward_backward} in Section \ref{section:convergence} formalizes this equivalence. We can turn the update equations into an algorithm as shown in algorithm \ref{algo:ptd}.

\begin{algorithm}[H]
\caption{Preferential TD: Linear FA}
\begin{algorithmic}[1]
    \label{PTD_online}
    \STATE Input: $\pi$,$\gamma$,$\param$, $\fvec$
    \STATE Initialize: $\w_{-1}=0, e_{-1}=0$
    \STATE Output: $\w_T$
    \FOR{$t:0 \to T$}
        \STATE Take action $a \sim \pi(s_t)$ , observe $r_{t+1},s_{t+1}$
        \STATE $v(s_t)=\w_t^T \fvec(s_{t}) , v(s_{t+1})= \w_t^T \fvec(s_{t+1})$ 
        \STATE $\delta_t =  r_{t+1} + \gamma v(s_{t+1}) - v(s_t)$
        \STATE $e_t = \param(s_t) \fvec(s_t) + \gamma (1 - \param(s_t)) e_{t-1}$
        \STATE $\w_{t+1} \gets \w_t + \alpha_t \delta_t e_t$
    \ENDFOR
\end{algorithmic}
\label{algo:ptd}
\end{algorithm}

\section{Convergence of PTD}
\label{section:convergence}

We consider a finite, irreducible, aperiodic Markov chain. Let $\{s_t | t = 0, 1, 2, \dots \}$ be the sequence of states visited by the Markov chain and let $\dis(s)$ denote the steady-state probability of $s\in \s$. 
We assume that $\dis(s)>0,\ \forall s \in \s$. Let $D^{\pi}$ be the $|\s|\times|\s|$ diagonal matrix with $D^{\pi}_{i,i} = \dis(i)$. We denote the Euclidean norm on vectors or Euclidean-induced norm on matrices by $\norm{\cdot}$: $\norm{A} = \max_{\norm{x}=1} \norm{Ax}$. We make the following assumptions to establish the convergence result \footnote{Similar assumptions are made to establish the convergence proof of TD($\lambda$) in the linear function approximation setting for a constant $\lambda$ \citep{tsitsiklis1997analysis}.}.

\begin{assumption}
\label{assump:feature_assumption}
The feature matrix, $\fmat \in \mathbb{R}^{|\s|} \times \mathbb{R}^{k}$ is a full column rank matrix, i.e., the column vectors $\{\fvec_i | i=1 \dots k\}$ are linearly independent. Also, $\norm{\fmat} \leq M$, where $M$ is a constant.
\end{assumption}

\begin{assumption}
\label{assump:rapid_mixing}
The Markov chain is rapidly mixing:
$$|\pp(s_t = s| s_0) - \dis(s)| \leq C \rho^t,\ \forall s_0 \in \s, \rho < 1,$$ where $C$ is a constant. 
\end{assumption}

\begin{assumption}
\label{assump:step_size}
The sequence of step sizes satisfies the Robbins-Monro conditions: $$\sum _{{t=0}}^{{\infty }}\alpha_{t}=\infty \quad {\text{and}} \quad \sum _{{t=0}}^{{\infty }}\alpha_{t}^{2}<\infty.$$
\end{assumption}

The update equations of Preferential TD (cf. equation \ref{eq:preferential_etrace}) can be written as:
$$\w_{t+1} \gets \w_t + \alpha_t (b(X_t) - A(X_t) \w_t ),$$ where $b(X_t) = \e_t r_{t+1}$, $A(X_t) = \e_t(\fvec(s_t) - \gamma \fvec(s_{t+1}))^T$, $X_t = (s_t, s_{t+1}, e_t)$ and $\e_t$ is the eligibility trace of PTD. Let $\keyA = \Edis[A(X_t)]$ and $\keyb = \Edis[b(X_t)]$. Let $\pbeta$ be a new transition matrix that accounts for the termination due to bootstrapping and discounting \footnote{$\pbeta$ is similar to ETD's $\pp^{\lambda}$ \citep{sutton2016emphatic}, $\param(s)$ takes the role of $1 - \lambda(s)$ and we consider a constant discount setting.}, defined as: $\pbeta = \gamma \Big(\sum_{k=0}^{\infty} (\gamma \pp (I-B))^k \Big) \pp B$. This is a sub-stochastic matrix for $\gamma<1$ and a stochastic matrix when $\gamma=1$. 

\begin{lemma}
\label{remark:keyA_keyB_definition}
The expected quantities $\keyA$ and $\keyb$ are given by $\keyA = \fmat^T D^\pi B (I - \pbeta)$ and $\keyb = \fmat^T D^\pi B (I - \gamma \pp)^{-1} \rp$. 
\end{lemma}
We use the proof template from Emphatic TD \citep{sutton2016emphatic} to get the desired result. The proof is provided in Appendix \ref{app:keyA_keyB_def_proof}.

We are now equipped to establish the forward-backward equivalence.
\begin{theorem}
\label{remark:forward_backward}
The forward and the backward views of PTD are equivalent in expectation: $$\keyb - \keyA \w = \fmat^T D \Big(\Tp (\fmat \w) - \fmat \w \Big).$$
\end{theorem}
The proof is provided in Appendix \ref{app:forward_backward_proof}.

The next two lemmas verify certain conditions on $\keyA$ that are required to show that it is positive definite.
\begin{lemma}
\label{remark:positive_row_sums}
$\keyA$ has positive diagonal elements with positive row sums.
\end{lemma}
\begin{proof}
Consider the term $(I - \pbeta)$. $\pbeta$ is a sub-stochastic matrix for $\gamma \in [0, 1)$, hence, $\sum_j [\pbeta]_{i, j} < 1$. Therefore, $\sum_j [I - \pbeta]_{i,j} = 1 - \sum_j [\pbeta]_{i, j} > 0$. Additionally, $[I - \pbeta]_{i,i} > 0$ since $[\pbeta]_{i,i} < 1$.
\end{proof}

\begin{lemma}
\label{remark:positive_column_sums}
The column sums of $\keyA$ are positive.
\end{lemma}
The idea of the proof is to show that $\dis B \pbeta = \dis B$ for $\gamma=1$. This implies $\dis B (I - \pbeta) = 0$ for $\gamma=1$ and $\dis B (I - \pbeta) > 0$ for $\gamma \in [0, 1)$ as $\dis B \pbeta < \dis B$. Therefore, column sums are positive. The complete proof is provided in Appendix \ref{app:col_sums_proof}.
Note that $\dis B$ is not the stationary distribution of $\pbeta$, as it is unnormalized.

\begin{lemma}
\label{lemma:keyA_PD}
$\keyA$ is positive definite.
\end{lemma}
\begin{proof}
From lemmas \ref{remark:positive_row_sums} and \ref{remark:positive_column_sums}, the row sums and the column sums of $\keyA$ are positive. Also, the diagonal elements of $\keyA$ are positive. Therefore, $\keyA + \keyA^T$ is a strictly diagonally dominant matrix with positive diagonal elements. We can use corollary 1.22 from \citet{varga1999matrix} (provided in Appendix \ref{app:std_corollary} for completeness) to conclude that $\keyA + \keyA^T$ is positive definite. Hence, $\keyA$ is also positive definite.
\end{proof}

The next lemma shows that the bias resulting from initial updates, i.e., when the chain has not yet reached the stationarity, is controlled.
\begin{lemma}
\label{lemma:sampling_bias}
We have, $\sum_{t=0}^{\infty} \norm{\Ep[A(X_t)|X_0] - \keyA} \leq C_1$ and $\sum_{t=0}^{\infty} \norm{\Ep[b(X_t)|X_0] - \keyb} \leq C_2$ under assumption \ref{assump:rapid_mixing}, where $C_1$ and $C_2$ are constants.
\end{lemma}
The proof is provided in Appendix \ref{app:sampling_bias_proof}.

We now present the main result of the paper.
\begin{theorem}
The sequence of expected updates computed by Preferential TD converges to a unique fixed point.
\end{theorem}
\begin{proof}
We can establish convergence by satisfying all the conditions of a standard result from the stochastic approximation literature, such as theorem 2 in \citet{tsitsiklis1997analysis} or proposition 4.8 in \citet{bertsekas1996neuro} (provided in Appendix \ref{app:std_theorem} for completeness). Assumption \ref{assump:step_size} satisfies the step-size requirement. Lemmas \ref{remark:keyA_keyB_definition} and \ref{lemma:keyA_PD} meet conditions 3 and 4. Lemma \ref{lemma:sampling_bias} controls the noise from sampling, hence satisfying the final requirement. Therefore, the expected update sequence of Preferential TD converges. The fixed point lies in the span of the feature matrix and this point has zero projected Bellman error. That is, $\proj \Big(\Tp(\fmat \wpi) - \fmat \wpi \Big) = 0$, where $\proj = \fmat (\fmat^T D^\pi \fmat)^{-1}\fmat^T D^\pi$ is the projection matrix.
\end{proof}

{\bf Consequences:} It is well-known that using state-dependent bootstrapping or re-weighing updates in TD($\lambda$) will result in convergence issues even in the on-policy setting. This was demonstrated using counterexamples \citep{mahmood2017incremental, ghiassian2017first}. Our result establishes convergence when a state-dependent bootstrapping parameter is used \textit{in addition to} re-weighing the updates. We analyze these examples below and show that they are invalid for Preferential TD. This makes Preferential TD one of only two algorithms to converge with standard assumptions in the function approximation setting, Emphatic TD being the other \citep{yu2015convergence}.

\textbf{Counterexample 1 \citep{mahmood2017incremental}:} Two state MDP with $\pp = \begin{bmatrix} 0.5 & 0.5\\ 0.5 & 0.5 \end{bmatrix}$, $\fmat = \begin{bmatrix} 0.5\\1 \end{bmatrix}$, $\lambda = \begin{bmatrix} 0.99 & 0.8 \end{bmatrix}$, $\dis = \begin{bmatrix} 0.5 & 0.5 \end{bmatrix}$, $\gamma = 0.99$. The key matrix of TD($\lambda$) is given by, $\keyA = \begin{bmatrix} -0.0429 \end{bmatrix}$, which is not positive definite. Therefore, the updates will not converge. The key matrix of Preferential TD for the same setup is $\keyA =  \begin{bmatrix} 0.009 \end{bmatrix}$ which is positive definite \footnote{We set $\lambda(s_0)=0.99$ instead of $1$, a small change to the original example, because $\param(s_0)=0$ when $\lambda(s_0)=1$ and such states can be excluded from the analysis.}. The second counterexample is analyzed in Appendix \ref{app_sec:counterexamples}.

\section{Related work}
A natural generalization of TD($\lambda$) is to make $\lambda$ a state-dependent function~\citep{sutton1995td, sutton1999between, sutton2016emphatic}. This setting produces significant improvements in many cases~\citep{sutton1995td, sutton2016emphatic, white2016greedy, xu2018meta}. However, it only modifies the target during bootstrapping, not the extent to which states are updated. The closest approach in spirit to ours is Emphatic TD($\lambda$) (Emphatic TD or ETD), which introduces a framework to reweigh the updates using the interest function ($i(s) \in \R^+, \forall s \in \s$) \citep{sutton2016emphatic}. Emphatic TD uses the interest of past states and the state-dependent bootstrap parameter $\lambda$ of the current state to construct the emphasis, $M_t \in \R^+$, at time $t$. The updates are then weighted using $M_t$. By this construction, the emphasis is a trajectory dependent quantity. Two different trajectories leading to a particular state have a different emphasis on the updates. This can be beneficial in the off-policy case, where one may want to carry along importance sampling ratios from the past. However, a trajectory-based quantity can result in high variance, as the same state could get different updates depending upon its past.

Emphatic TD and Preferential TD share the idea of reweighing the updates based on the agent's preference for states. Emphatic TD uses a trajectory-based quantity, whereas Preferential TD uses a state-dependent parameter to reweigh the updates. Furthermore, Preferential TD uses a single parameter to update and bootstrap, instead of two in Emphatic TD. Our analysis in Section~\ref{section:convergence} suggests that the two algorithms' fixed points are also different. However, we can set up Emphatic TD to achieve a similar effect to PTD (though not identical) when $\param \in \{0, 1\}$, by setting $\lambda = 1-\param$ and the interest of a state proportional to the preference for that state. In our experiments, we use this setting as well, which will be denoted as ETD-variable algorithm.

The fact that preference is a state-dependent quantity fits well with the intuition that if a state is partially observable, we may always want to avoid involving it in bootstrapping, regardless of the trajectory on which it occurs. Preferential TD lies in between TD($\lambda$) and Emphatic TD. Like TD($\lambda$), it uses a single state-dependent bootstrap function. Like Emphatic TD, it reweighs the updates but using only the preference.

Other works share the idea of ignoring updates on partially observable states, e.g.~\citealt{xu2017natural, thodoroff2019recurrent}. They use a trajectory-based value as a substitute to the value of a partially observable state. Temporal value transport~\citep{hung2019optimizing} uses an attention mechanism to pick past states to update, bypassing partially observable states for credit assignment. However, the theoretical properties of trajectory-dependent values or attention-based credit assignment are poorly understood at the moment. Our method is unbiased, and, as seen from Section \ref{section:convergence}, it can be understood well from a theoretical standpoint.

Our ideas are also related to the idea of a backward model~\citet{chelu2020forethought}, in which an explicit model of precursor states is constructed and used to propagate TD-errors in a planning-style algorithm, with a similar motivation of avoiding partially observable states and improving credit assignment. However, our approach is model-free.

\section{Illustrations}
In this section, we test Preferential TD on policy evaluation tasks in four different settings\footnote{Code to reproduce the results can be found \href{https://github.com/NishanthVAnand/Preferential-Temporal-Difference-Learning}{here}.}: tabular, linear, semi-linear (linear predictor with non-linear features), and non-linear (end-to-end training). Note that our theoretical results do not cover the last setup; however, it is quite easy to implement the algorithm in this setup.

\begin{figure*}[ht]
    \centering
    \includegraphics[width=0.85\textwidth, height=4.5cm]{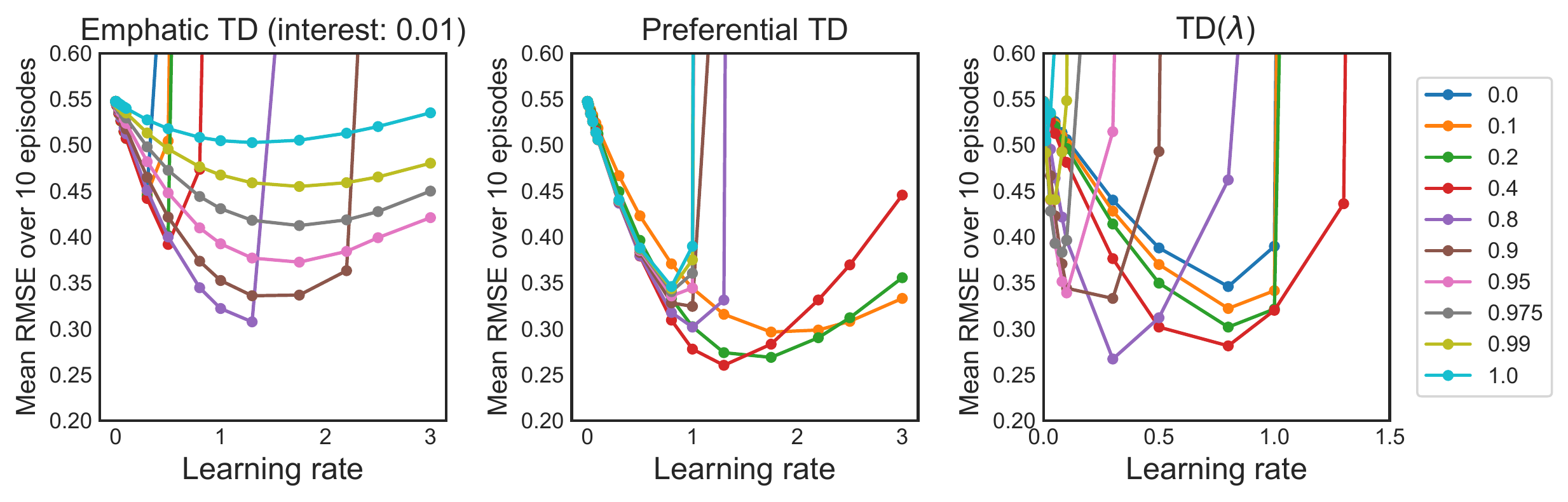}
    \caption{Root Mean Square Error (RMSE) of Emphatic TD, TD($\lambda$), and Preferential TD as a function of learning rate $\alpha$. Different curves in the plot correspond to different values of $\lambda$ or $\param$ (depending upon the algorithm).}
    \label{fig:tabular_hp_plot1}
\end{figure*}

\subsection{Tabular}
In this setting, the value function is represented as a look-up table. This experiment is provided to assist in understanding PTD as a whole in the limited sample setting for various choices of constant preference function.

\textbf{Task description:} We consider the 19-state random walk problem from \citealt{sutton1988learning, sutton2018reinforcement}. In this MDP, there are 19 states connected sequentially, with two terminal states on either side. From each state, the agent can choose to go left or right. The agent gets a reward of $+1$ when it transitions to the rightmost terminal state and $-1$ in the leftmost state; otherwise, rewards are $0$. We set $\gamma=1$ and the policy to uniformly random. We compare Emphatic TD, TD($\lambda$), and PTD on this task. We consider fixed parameters for these methods, i.e., interest in ETD is fixed, $\lambda$ in ETD and TD($\lambda$) are fixed, and $\param$ in PTD is fixed. For ETD, we selected a constant interest of $0.01$ on all states (selected from $\{0.01, 0.05, 0.1, 0.25\}$ based on hyperparameter search) for each $\lambda$-$\alpha$ pair. We consider fixed parameters to understand the differences between the three algorithms.

\textbf{Observations:} The U-curves obtained are presented in Figure \ref{fig:tabular_hp_plot1}, which depicts Root Mean Square Error (RMSE) on all states, obtained in the first 10 training episodes, as a function of $\alpha$. Various curves correspond to different choices of fixed $\lambda$ (or $\param$). In TD($\lambda$), when $\lambda$ is close to 0, the updates are close to TD(0), and hence there is high bias. When $\lambda$ is close to 1, the updates are close to Monte-Carlo, and there is a high variance. A good bias-variance trade-off is obtained for intermediate values of $\lambda$, resulting in low RMSE. A similar trade-off can be observed in PTD; however, the trade-off is different from TD($\lambda$). A $\param$ value close to $1$ results in a TD($0$)-style update (similar to $\lambda=0$), but $\param \approx 0$ results in negligible updates (as opposed to a Monte-Carlo update for similar values of $\lambda$). The resulting bias-variance trade-off is better.

TD($\lambda$) is very sensitive to the value of the learning rate. When $\lambda$ gets close to $1$, the RMSE is high even for a low learning rate, resulting in sharp U-curves. This is because of the high variance in the updates. However, this is not the case for PTD, which is relatively stable for all values of $\param$ on a wide range of learning rates. PTD exhibits the best behaviour when using a learning rate greater than $1$ for some values of $\param$. This behaviour can be attributed to negligible updates performed when a return close to Monte-Carlo ($\param$ is close to 0) is used. Due to negligible updating, the variance in the updates is contained.

The performance of ETD was slightly worse than the other two algorithms. ETD is very sensitive to the value of the interest function. We observed a large drop in performance for very small changes in the interest value. ETD also performs poorly with low learning rates when the interest is set high. This behaviour can be attributed to the high variance caused by the emphasis.

\subsection{Linear setting}
\textbf{Task description:} We consider two toy tasks to evaluate PTD with linear function approximation.

In the first task, depicted in Figure \ref{fig:linear_mdp_a}, the agent starts in $S1$ and can pick one of the two available actions, $\{up, down\}$, transitioning onto the corresponding chain. A goal state is attached at the end of each chain. Rewards of $+2$ and $-1$ are obtained upon reaching the goal state on the upper and lower chains, respectively. The chain has a sequence of partially observable states (corridor). In these states, both actions move the agent towards the goal by one state.

For the second task, we attach two MDPs of the same kind as used in task 1, as shown in Figure~\ref{fig:linear_mdp_b}. The actions in the connecting states $S1, S2, S3$ transition the agent to the upper and lower chains, but both actions have the same effect in the partially observable states (corridor). A goal state is attached at the end of each chain. To make the task challenging, a stochastic reward (as shown in Figure~\ref{fig:linear_mdp_b}) is given to the agent. In both tasks, the connecting states and the goal states are fully observable and have a unique representation, while the states in the corridor are represented by Gaussian noise, $\mathcal{N}(0.5,1)$.

\begin{figure}[ht]
    \centering
    \begin{subfigure}{0.48\textwidth}
        \centering
        \includegraphics[width=\textwidth, height=3cm]{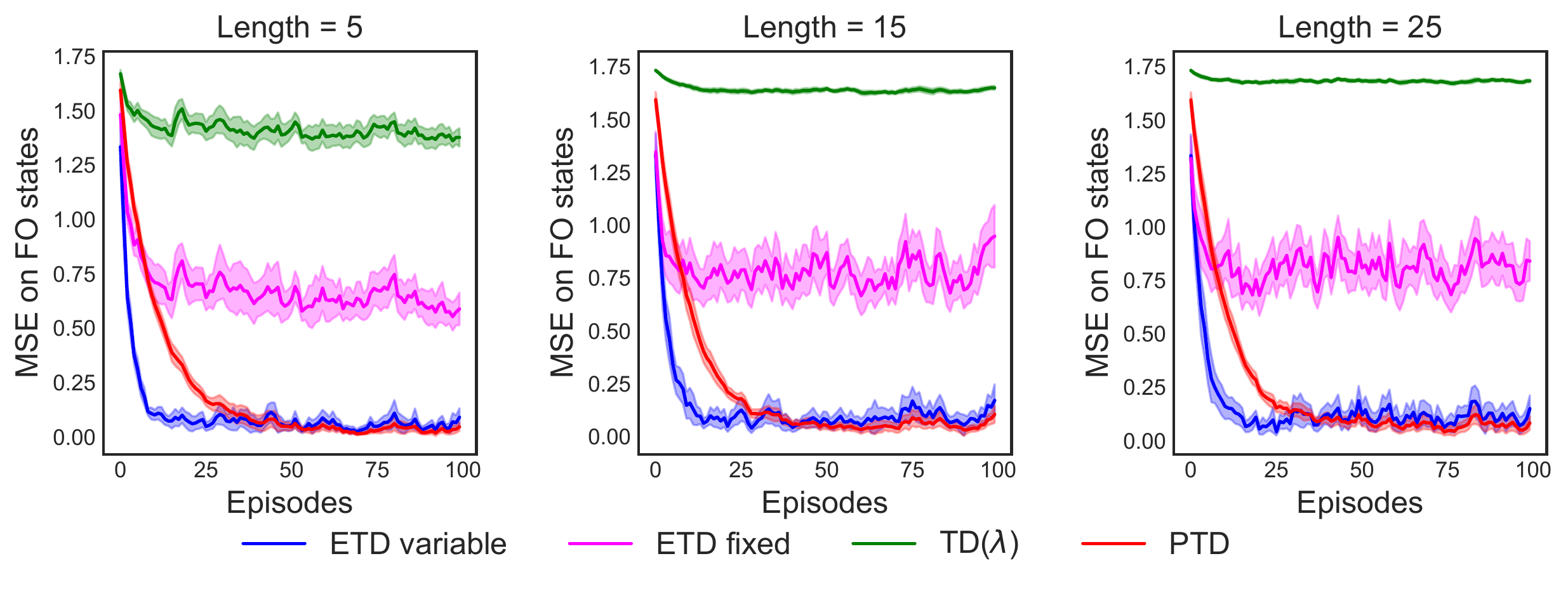}
        \caption{Task 1 results}
        \label{fig:task1_learning_linear}
    \end{subfigure}
    \hfill
    \begin{subfigure}{0.48\textwidth}
        \centering
        \includegraphics[width=\textwidth, height=3cm]{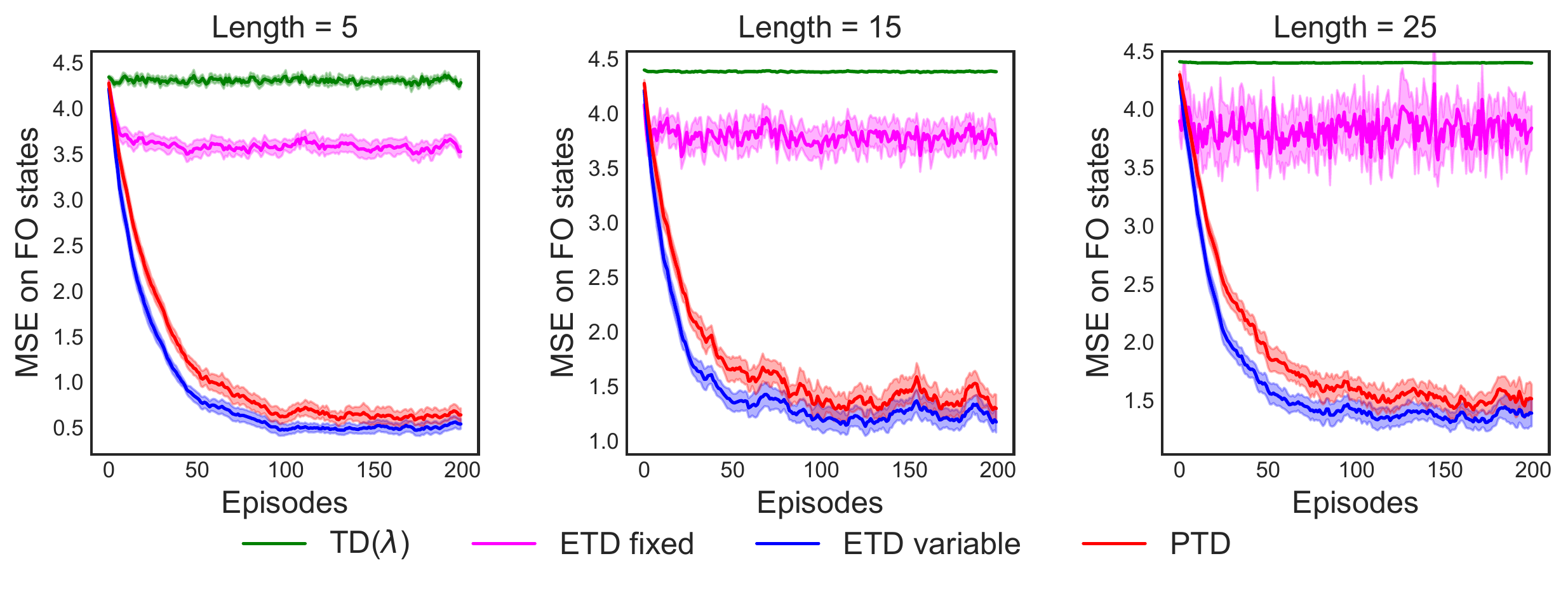}
        \caption{Task 2 results}
        \label{fig:task2_learning_linear}
    \end{subfigure}
    \caption{The mean squared error of fully observable states' values plotted as a function of episodes for various algorithms. Different plots correspond to different lengths of the corridor.}
    \label{fig:linear_learning}
\end{figure}

\textbf{Setup:} We tested 4 algorithms: TD($\lambda$), PTD, and two versions of ETD. The first ETD version has a fixed interest on all the states (referred to as ETD-fixed), and the second variant has a state-dependent interest (referred to as ETD-variable). We set the preference on fully observable states to $1$ and partially observable states to $0$. For a fair comparison, we make $\lambda$ a state-dependent quantity by setting the values of the fully observable states to $0$ and $1$ on partially observable states for TD($\lambda$) and ETD. In the case of ETD-fixed, we set the interest of all states to $0.01$ (selected from $\{0.01, 0.05, 0.1\}$ based on the hyperparameter search, see Appendix \ref{app_subsec:linear_experiments} for details). For ETD-variable, we set the interest to $0.5$ for fully observable states and $0$ on all other states. This choice results in an emphasis of $1$ on fully observable states. ETD-variable yielded very similar performance with various interest values but with different learning rates. We varied the corridor length ($\{5, 10, 15, 20, 25\}$) in our experiments. For each length and algorithm, we chose an optimal learning rate from 20 different values. We used $\gamma=1$, a uniformly random policy, and the value function is predicted only at fully observable states. The results are averaged across 25 seeds, and a confidence interval of $95\%$ is used in the plots\footnote{We use the t-distribution to compute the confidence intervals in all the settings.}.

\textbf{Observations:} A state-dependent $\lambda$ (or $\param$) results in bootstrapping from fully observable states only. Additionally, PTD weighs the updates on various states according to $\param$. The learning curves on both tasks are reported in Figures \ref{fig:linear_learning} and \ref{app_fig:linear_learning}. TD($\lambda)$ performs poorly, and the errors increase with the length of the chain on both tasks. Also, the values of the fully observable states stop improving even as the number of training episodes increases. This is to be expected, as TD($\lambda)$ updates the values of all states. Thus, errors in estimating the partially observable states affect the predictions at the fully observable states due to the generalization in the function approximator.

ETD-fixed performs a very small update on partially observable states due to the presence of a small interest. Nevertheless, this is sufficient to affect the predictions on fully observable states. The error in ETD-fixed increases with a small increase in the interest. ETD-fixed is also highly sensitive to the learning rate. However, ETD-variable performs the best. With the choice of $\lambda$ and interest function described previously, ETD-variable ignores the updates on corridor states and, at the same time, does not bootstrap from these states, producing a very low error in both tasks. This effect is similar to PTD. However, bootstrapping and updating are controlled by a single parameter ($\param$) in PTD, while ETD-variable has two parameters, $\lambda$ and the interest function, and we have to optimize both. Hence, while it converges slightly faster than PTD, this is at the cost of increased tuning complexity.

\subsection{Semi-linear setting}
\label{subsec:semi_linear}

\textbf{Task description:} We consider two grid navigation tasks shown in figures \ref{fig:task1_mdp_semi} and \ref{fig:task2_mdp_semi}. In the first task, the states forming a cross at the center are partially observable. In the second task, $50\%$ of the states are partially observable (sampled randomly for each seed). The agent starts randomly in one of the states in the top left quadrant, and it can choose from four actions ($up, down, left, right$), moving in the corresponding direction by one state. A reward of $+5$ is obtained when the agent transitions to the state at the bottom right corner. The transitions are deterministic and $\gamma=0.99$. To generate the features for these tasks, we first train a single-layer neural network to predict the value function using Monte Carlo targets on the fully observable grid. We use the output of the penultimate layer of the trained network to get the features for the states. We then train a linear function approximator on top of these features to predict the value function of a uniformly random policy. We use three grid sizes ($\{8\times8, 12\times12, 16\times16\}$) for experimentation.

\textbf{Setup:} We consider the same four algorithms in this task. We set $\lambda=0$ and $\param=1$ on fully observable states, and set $\lambda=1$ and $\param=0$ on other states. We set the interest to $0.01$ for ETD-fixed (selected from $\{0.01, 0.05, 0.1\}$ based on the hyperparameter search, see Appendix \ref{app_subsec:semi-linear}). The interest is set to $0.5$ on fully observable states and $0$ on other states for ETD-variable. The results are averaged across 50 seeds, and a confidence interval of $50\%$ is used in the plots\footnote{Picked for clarity of the plotting}.

\begin{figure}[ht]
    \centering
    \begin{subfigure}{0.23\textwidth}
        \centering
        \includegraphics[width=0.95\textwidth, height=0.95\textwidth]{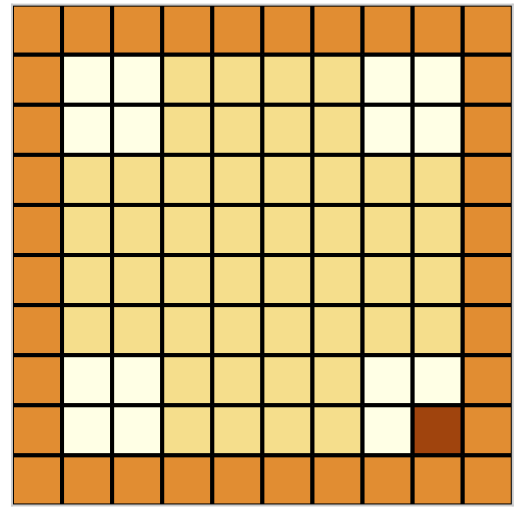}
        \caption{Grid task 1}
        \label{fig:task1_mdp_semi}
    \end{subfigure}
    \hfill
    \begin{subfigure}{0.23\textwidth}
        \centering
        \includegraphics[width=0.95\textwidth, height=0.95\textwidth]{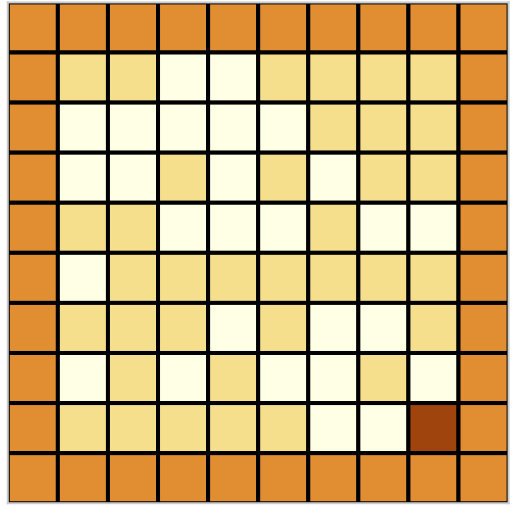}
        \caption{Grid task 2}
        \label{fig:task2_mdp_semi}
    \end{subfigure}
    \hfill
        \begin{subfigure}{0.23\textwidth}
        \centering
        \includegraphics[width=\textwidth, height=0.95\textwidth]{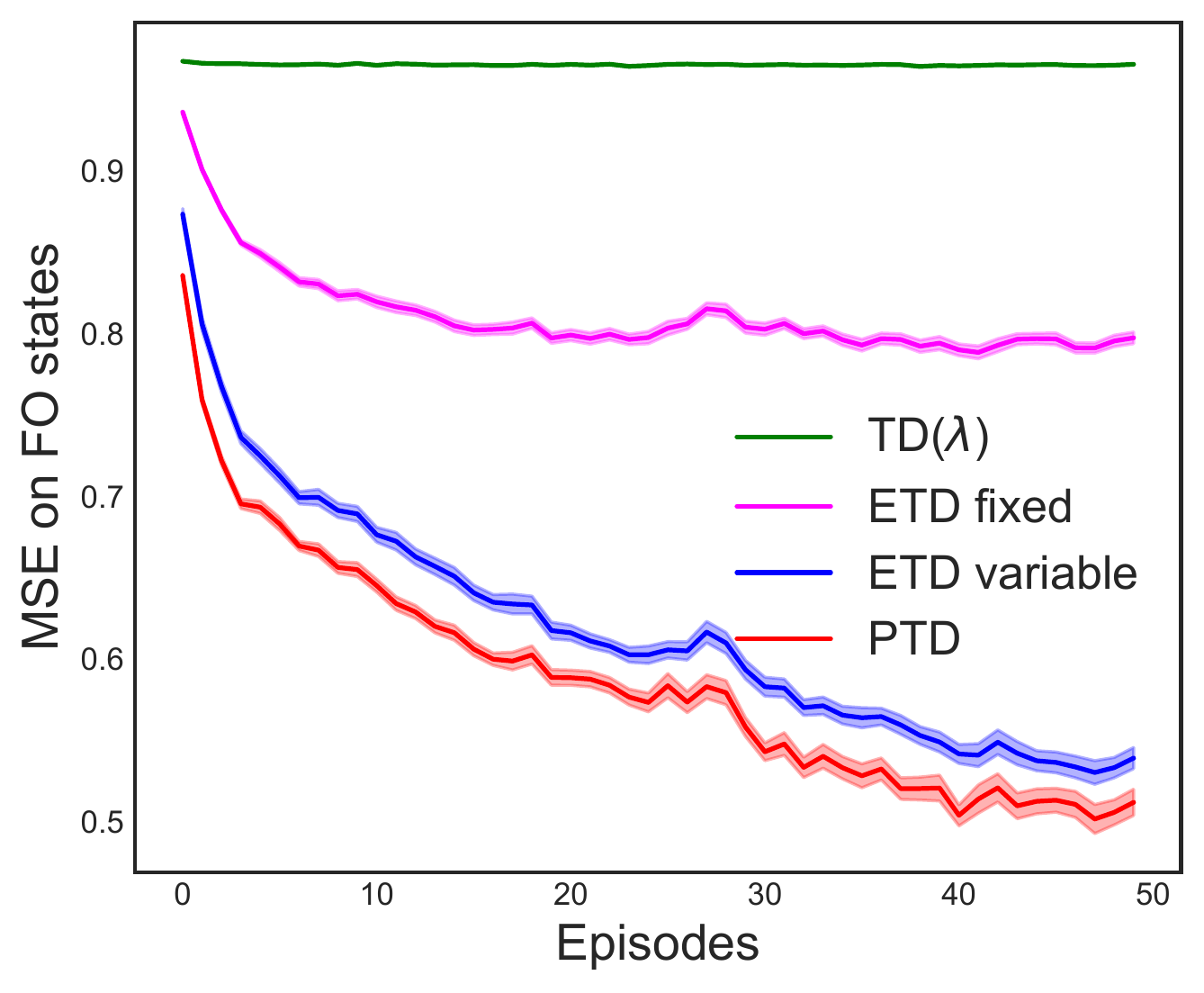}
        \caption{Grid task 1 results}
        \label{fig:task1_learning_semi}
    \end{subfigure}
    \hfill
    \begin{subfigure}{0.23\textwidth}
        \centering
        \includegraphics[width=\textwidth, height=0.95\textwidth]{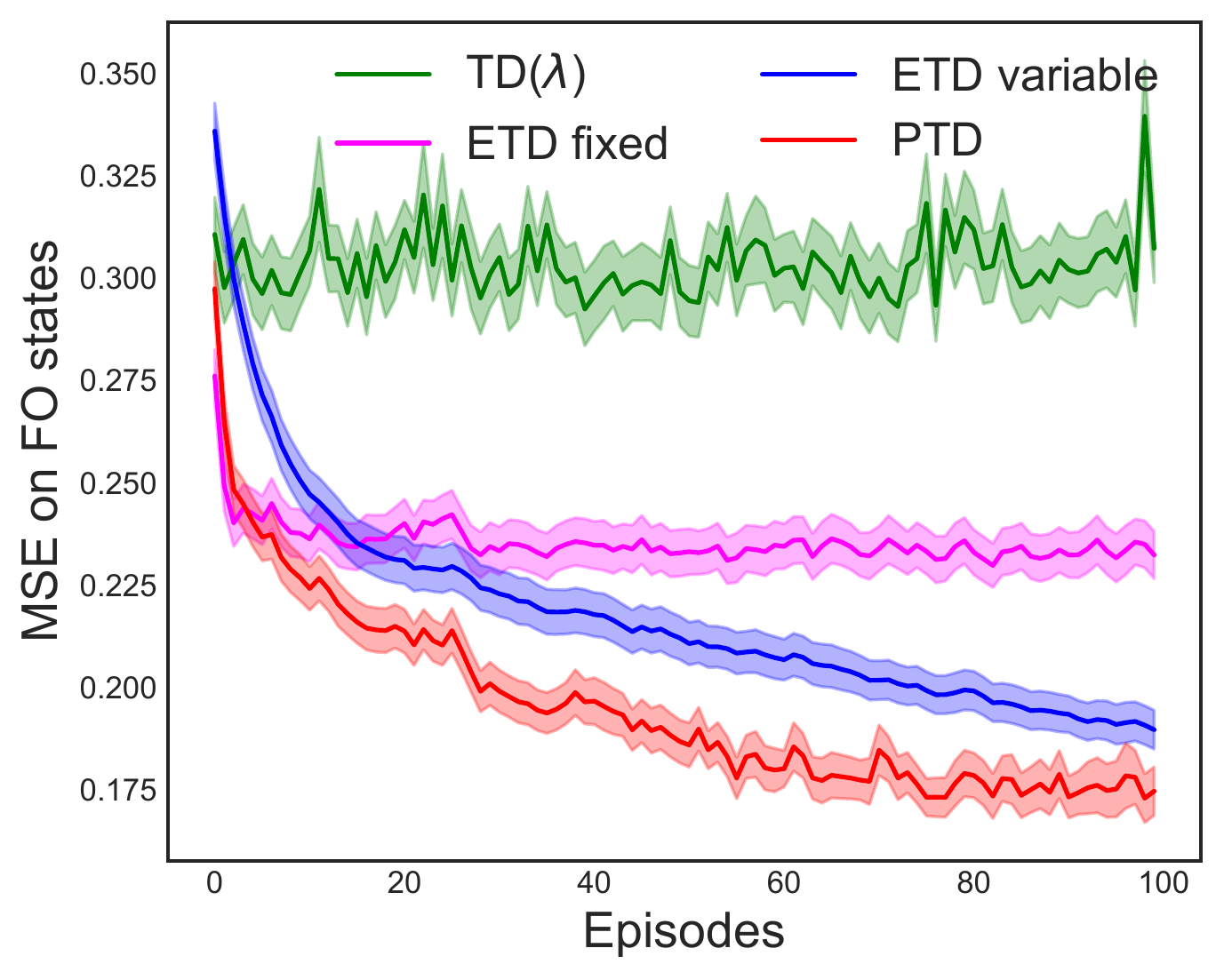}
        \caption{Grid task 2 results}
        \label{fig:task2_learning_semi}
    \end{subfigure}
    \caption{(\textit{Top}) \textbf{Grid tasks:} The dark brown square is the terminal state. The light squares are fully observable and the darker squares are partially observable. The walls are indicated in orange. (\textit{Bottom}) The mean squared error of fully observable states' values as a function of the number of episodes for various algorithms on $16 \times 16$ grid.}
    \label{fig:semilinear_learning}
\end{figure}

\textbf{Observations:} We report the mean squared error of the value function at fully observable states as the performance measure. The learning curves on both tasks, as a function of number of episodes are reported in figures \ref{fig:task1_learning_semi}, \ref{fig:task2_learning_semi}, and \ref{app_fig:semi_learning}. TD($\lambda)$ performed the worst because of the reasons mentioned earlier. Interestingly, ETD-fixed performed much better in this setting compared to the previous one. Its performance dropped significantly as the grid size increased in task 1, but the drop was marginal for task 2. ETD-variable performed much better than ETD-fixed. The performance on both tasks for grid size $8\times8$ and $12\times12$ was the same as PTD. However, there was a slight drop in performance for grid size $16\times16$. This is because, in the larger tasks, the trajectories can vary widely. As a result, the update on a particular state is reweighed differently depending upon the trajectory, causing variance. The hyperparameter tuning experiments also showed that ETD-variable is sensitive to small changes in learning rate. PTD performs best on all grid sizes and both tasks.

\subsection{Non-linear setting}
\label{subsec:non_linear}

\textbf{Task description:} The aim of this experiment is two-fold: (1) to verify if PTD still learns better predictions when the capacity of the function approximator is decreased, and (2) to verify if PTD is compatible with end-to-end training of a non-linear function approximator. We use the same grid tasks described in section \ref{subsec:semi_linear} for experimentation.

\subsubsection{Forward view (Returns)}
\label{subsubsec:forward}

\begin{figure}[ht]
    \centering
    \begin{subfigure}{0.48\textwidth}
        \centering
        \includegraphics[width=\textwidth, height=3cm]{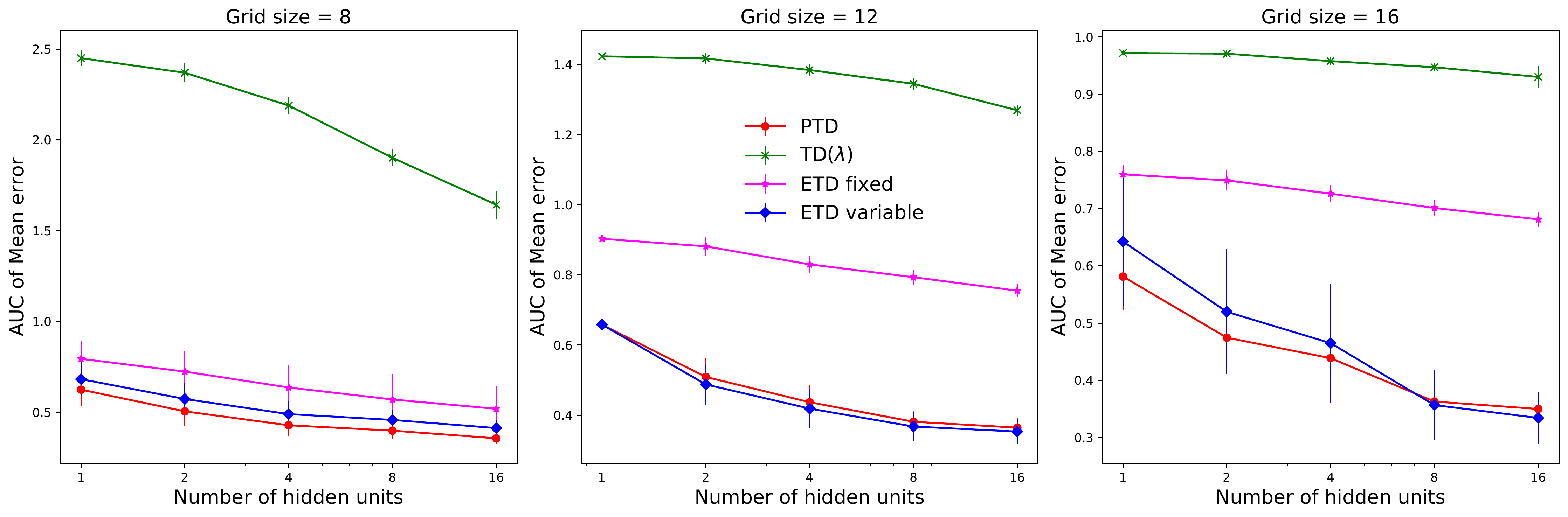}
        \caption{Grid task 1 results}
        \label{fig:Grid1_endtoend}
    \end{subfigure}
    \begin{subfigure}{0.48\textwidth}
        \centering
        \includegraphics[width=\textwidth, height=3cm]{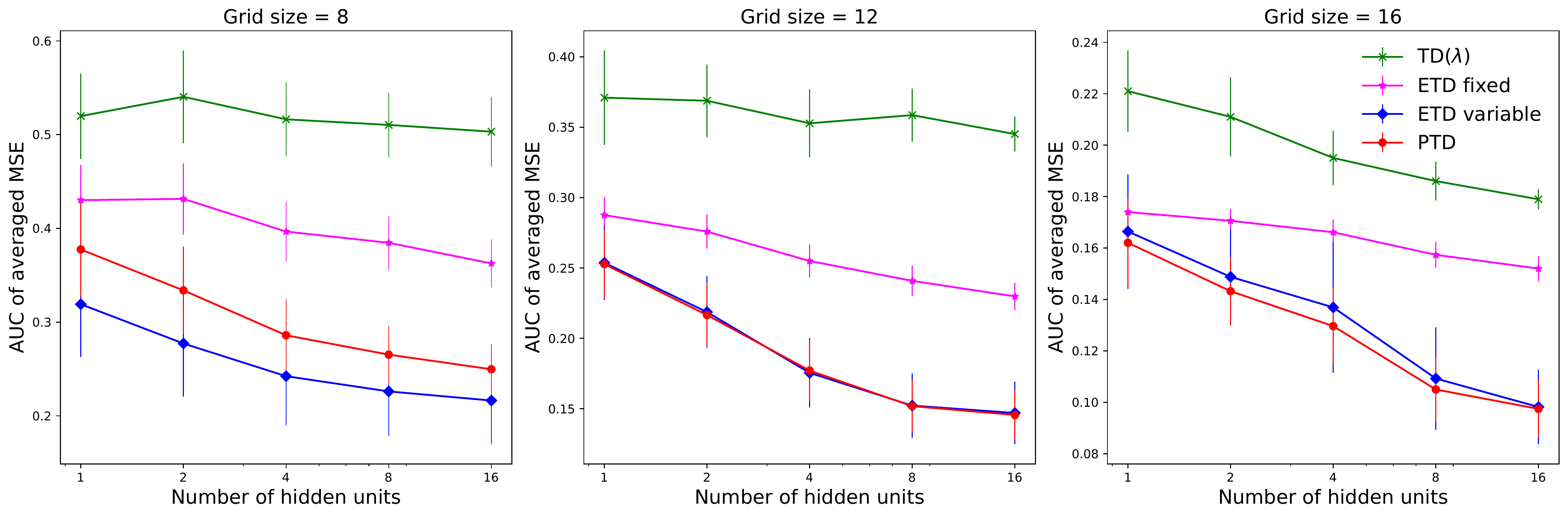}
        \caption{Grid task 2 results}
        \label{fig:Grid2_endtoend}
    \end{subfigure}
    \caption{The area under the curve (AUC) of mean squared error of the values of the fully observable states, averaged across seeds and episodes as a function of the number of hidden units for various algorithms. The error bars indicate the confidence interval of $90\%$ across 50 seeds. Different plots corresponds to different sizes of the grid task.}
    \label{fig:Grids_endtoend}
\end{figure}

\textbf{Setup:} We consider the same four algorithms and set $\lambda$, $\param$ and interest values as described in the previous section. We use a single-layered neural network whose input is a one-hot vector (size $n \times n$) indicating the agent's location in the grid, if the state is fully observable, and Gaussian noise otherwise. We used ReLU non-linearity in the hidden layer. We experimented with $\{1, 2, 4, 8, 16 \}$ hidden units to check if PTD can estimate the values when the approximation capacity is limited. We searched for the best learning rate for each algorithm and network configuration from a range of values. We picked the best learning rate based on the area under the error curve (AUC) of the MSE averaged across episodes and 25 seeds. We used 250 episodes for all the tasks. We trained the networks using the forward-view returns for all algorithms.

\textbf{Observations:} The results are presented in Figure \ref{fig:Grids_endtoend}. Each point in the plot is the AUC of the MSE for an algorithm and neural network configuration pair averaged across episodes and seeds. The error bars indicate the confidence interval of 90\% across seeds. TD($\lambda$) performs poorly because of the reasons stated before. ETD-fixed performs slightly better than TD($\lambda$), but the error is still high. PTD performs better than the other algorithms on both tasks and across all network sizes. The error is high only in the $16 \times 16$ grid for really small networks, with four or fewer hidden units. This is because the number of fully observable states is significantly larger for this grid size, and the network is simply not large enough to provide a good approximation for all these states. Nevertheless, PTD's predictions are significantly better than other algorithms in this setting. The behavior of ETD-variable is similar, but the variance is much higher. As before, hyperparameter tuning experiments indicate that ETD is also much more sensitive to learning rates.

\subsubsection{Backward view (Eligibility Traces)}

In this experiment, we are interested in finding out how PTD's eligibility traces would fare in the non-linear setting.

\begin{figure}[ht]
    \centering
    \begin{subfigure}{0.48\textwidth}
        \centering
        \includegraphics[width=\textwidth, height=3cm]{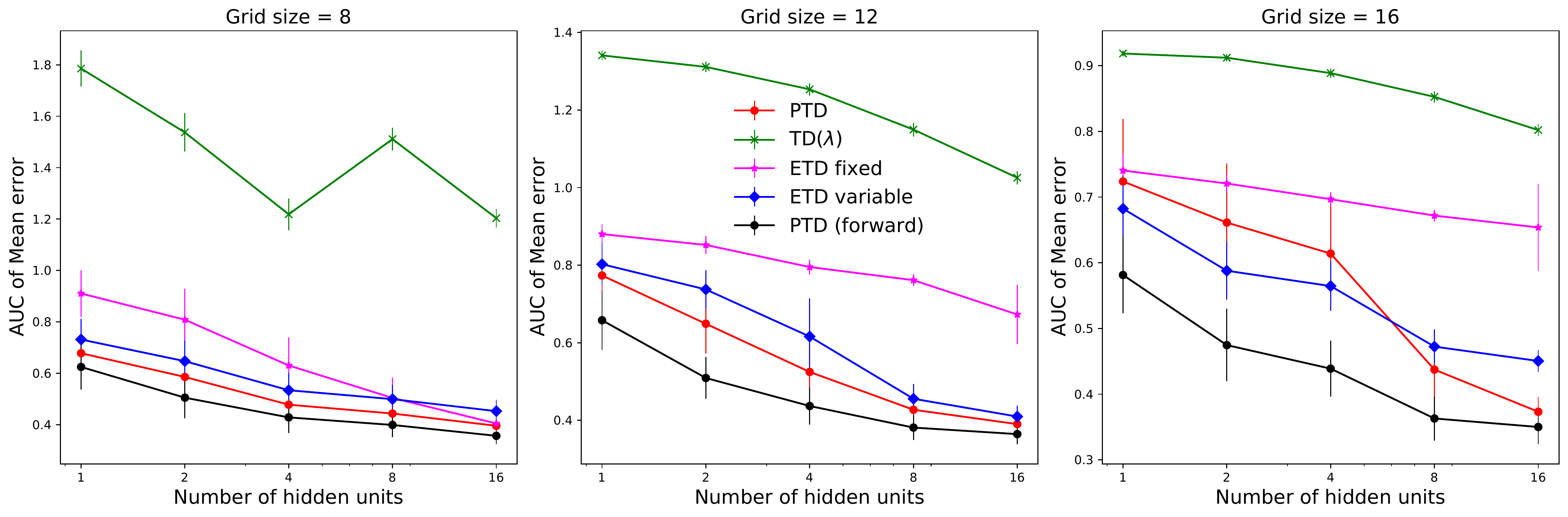}
        \caption{Grid task 1 results}
        \label{fig:Grid1_endtoend_traces}
    \end{subfigure}
    \begin{subfigure}{0.48\textwidth}
        \centering
        \includegraphics[width=\textwidth, height=3cm]{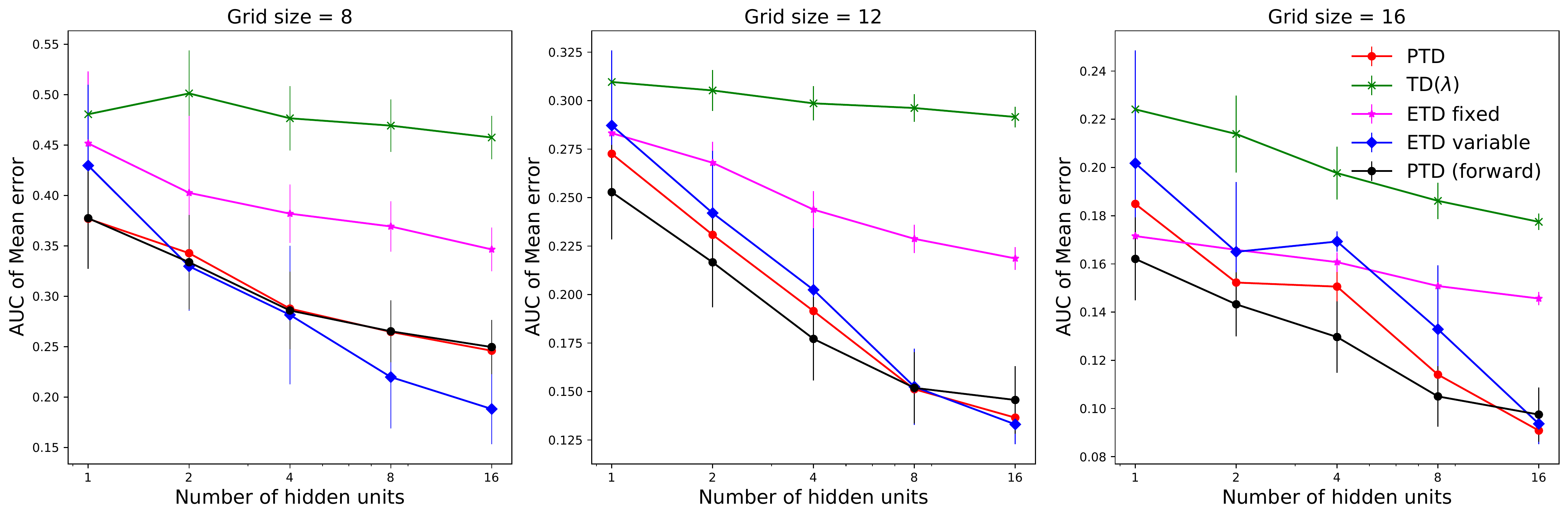}
        \caption{Grid task 2 results}
        \label{fig:Grid2_endtoend_traces}
    \end{subfigure}
    \caption{The area under the curve (AUC) of mean squared error of the values of the fully observable states, averaged across seeds and episodes as a function of the number of hidden units for various algorithms. The error bars indicate the confidence interval of $90\%$ across 50 seeds. Different plots corresponds to different sizes of the grid task.}
    \label{fig:Grids_endtoend_traces}
\end{figure}

\textbf{Setup:} The set of algorithms tested, network architecture, and the task details are same as Section \ref{subsubsec:forward}. We also followed the same procedure as the previous section to select the learning rate. We trained the networks in an online fashion using eligibility traces for all algorithms.

\textbf{Observations:} Eligibility traces facilitate online learning, and they provide a way to assign credit to the past states, making it a valuable tool in theory. However, the prediction errors of all the algorithms are higher than their forward view counterparts. The drop in performance is because eligibility traces remember the gradient that is computed with respect to the past parameters, which introduces significant bias. Besides, the bias in the eligibility traces caused instabilities in the training procedure of PTD and ETD-variable for high learning rates. The results are presented in Figure \ref{fig:Grids_endtoend_traces}. As shown, the performances of TD($\lambda$) and ETD-fixed performs are still poor compared to PTD and ETD-variable. Backward views of PTD and ETD-variable have slightly more error and variance across the seeds than their forward view equivalents when the capacity of the function approximator is small. However, they match the forward view performance when the capacity is increased. PTD still performs slightly better than ETD-variable in the $16 \times 16$ grids, indicating the usefulness of its eligibility traces with neural networks.

\section{Conclusions and Future Work}

We introduced Preferential TD learning, a novel TD learning method that updates the value function and bootstraps in the presence of state-dependent preferences, which allow these operations to be prioritized differently in different states. Our experiments show that PTD compares favourably to other TD variants, especially on tasks with partial observability. However, partial observability is not a requirement to use our method. PTD can be useful in other settings where updates can benefit from re-weighting. For instance, in an environment with \textit{bottleneck} states, having a high preference on such states could propagate credit faster. Preliminary experiments in Appendix \ref{app_sec:actor_critic} corroborate this claim, and further analysis could be interesting.

We set the preference function from the beginning to a fixed value in our experiments, but an important direction for future work is to learn or adapt it based on data. The preference function plays a dual role: its presence in the targets controls the amount of bootstrapping from future state values, and its presence in the updates determines the amount of re-weighting. Both can inspire learning approaches. The bootstrapping view opens the door to existing techniques such as meta-learning~\citep{white2016greedy, xu2018meta, zhao2020meta}, variance reduction~\citep{kearns2000bias, downey2010temporal}, and gradient interference~\citep{bengio2020interference}. Alternatively, we can leverage methods that identify partially observable states or important states (e.g. bottleneck states)~\citep{mcgovern2001automatic, stolle2002learning, bacon2017option, harutyunyan2019hindsight} to learn the preference function. We also believe that PTD would be useful in transfer learning, where one could learn parameterized preferences in a source task and use them in the updates on a target task to achieve faster training.

We demonstrated the utility of preferential updating in on-policy policy evaluation. The idea of preferential updating could also be exploited in other RL settings, such as off-policy learning, control, or policy gradients, to achieve faster learning. Our algorithm could also be applied to learning General Value Functions~\citep{comanici2018knowledge, sutton2017horde}, an exciting future direction.

As discussed earlier, our method bridges the gap between Emphatic TD and TD($\lambda$). Eligibility traces in PTD propagate the credit to the current state based on its preference, and the remaining credit goes to the past states. This idea is similar to the gradient updating scheme in the backpropagation through time algorithm, where the \textit{gates} in recurrent neural networks control the flow of credit to past states. We suspect an interesting connection between the eligibility trace mechanism in PTD and backpropagation through time. In the binary preference setting ($\param \in \{0, 1\}$), our method completely discards zero preference states in the MDP from bootstrapping and updating. The remaining states, whose preference is non-zero, participate in both. This creates a level of \textit{state abstraction} in the given problem.

Our ultimate goal is to train a small agent capable of learning in a large environment, where the agent can't represent values correctly everywhere since the state space is enormous~\citep{silver2021reward}. A logical step towards this goal is to find ways to effectively use function approximators with limited capacity compared to the environment. Therefore, it should prefer to update and bootstrap from a few states well rather than poorly estimate the values for all states. PTD helps achieve this goal and opens up avenues for developing more algorithms of this flavour.

\section*{Acknowledgements}
This research was funded through grants from NSERC and the CIFAR Learning in Machines and Brains Program. We thank Ahmed Touati for the helpful discussions on theory; and the anonymous reviewers and several colleagues at Mila for the constructive feedback on the draft.


\bibliography{example_paper}
\bibliographystyle{icml2021}

\newpage
\onecolumn
\appendix

\newtheorem{theorem*}{Theorem}[section]
\newtheorem{assumption*}{Assumption}[section]
\newtheorem{definition*}{Definition}[section]
\newtheorem{lemma*}{Lemma}[section]
\newtheorem{remark*}{Remark}[section]
\newtheorem{corollary*}{Corollary}[section]

\setcounter{section}{0}
\setcounter{lemma}{0}
\setcounter{remark}{0}
\setcounter{theorem}{0}

\counterwithin{figure}{section}
\counterwithin{table}{section}

\section{Appendix}

\subsection{Proofs}
\begin{theorem*}
\label{app:operator_proof}
The expected target of PTD's forward view can be summarized using the operator $$\Tp {\bf v} = B (I - \gamma \pp (I-B))^{-1} (\rp + \gamma \pp B {\bf v}) + (I-B){\bf v},$$
where $B$ is the $|\s| \times |\s|$ diagonal matrix with $\param(s)$ on its diagonal and $\rp$ and $\pp$ are the state reward vector and state-to-state transition matrix for policy $\pi$.
\end{theorem*}
\begin{proof}
The expected target of Preferential TD in the vector form is given by,
\begin{equation*}
    \Tp {\bf v} = (I-B) v + B \Big( \rp + \gamma \pp B {\bf v} + \gamma \pp (I-B) \rp + \gamma^2 \pp (I-B) \pp B {\bf v} + \dots \Big).
\end{equation*}
We can now express the reward terms and the value terms compactly by using the Neumann series expansion.
\begin{align*}
    \Tp {\bf v} &= (I-B) {\bf v} + B \Big(\rp + \gamma \pp (I-B) \rp + (\gamma \pp (I-B))^2 \rp + \dots \\
    &\qquad + \gamma \pp B {\bf v} + \gamma^2 \pp (I-B) \pp B {\bf v} + \gamma^3 (\pp (I-B))^2 \pp B {\bf v} + \dots \Big),\\
    &= (I-B) {\bf v} + B \Big((I - \gamma \pp (I-B))^{-1} \rp + \gamma (I - \gamma \pp (I-B))^{-1} \pp B {\bf v} \Big), \\
    &= B (I - \gamma \pp (I-B))^{-1} (\rp + \gamma \pp B {\bf v}) + (I-B) {\bf v}.
\end{align*}
\end{proof}

\begin{lemma*}
\label{app:keyA_keyB_def_proof}
The expected quantities $\keyA$ and $\keyb$ are given by $\keyA = \fmat^T D^\pi B (I - \pbeta)$ and $\keyb = \fmat^T D^\pi B (I - \gamma \pp)^{-1} \rp$.
\end{lemma*}
\begin{proof}
\begin{align*}
    \keyA &= \lim_{t\to \infty} \Ep[\keyA(X_t)], \\
    &= \lim_{t\to \infty}\Ep \Big[e_t \Big(\fvec(s_t) - \gamma \fvec(s_{t+1}) \Big)^T \Big], \\
    &= \sum_{s} \dis(s) \lim_{t\to \infty} \Ep \Big[e_t \Big(\fvec(s_t) - \gamma \fvec(s_{t+1}) \Big)^T \Big| s_t = s \Big], \\
    &= \sum_{s} \dis(s) \lim_{t\to \infty} \Ep \Big[\underbrace{\Big(\gamma (1-\param(s_t)) e_{t-1} + \param(s_t) \fvec(s_t)\Big) \Big(\fvec(s_t) - \gamma \fvec(s_{t+1})\Big)^T}_{\text{Independent due to conditioning on $s$}} \Big| s_t = s \Big], \\
    &= \sum_{s} \dis(s) \lim_{t\to \infty} \Ep \Big[\gamma (1-\param(s_t)) e_{t-1} + \param(s_t) \fvec(s_t) \Big| s_t = s \Big]\ \Ep \Big[ \Big(\fvec(s_t) - \gamma \fvec(s_{t+1}) \Big)^T \Big| s_t = s \Big], \\
    &= \sum_{s} e(s) \Big(\fvec(s) - \gamma \sum_{s'} \pp(s'|s) \fvec(s')\Big)^T,
\end{align*}
where $e(s) = \dis(s) \lim_{t\to \infty} \Ep[\gamma (1-\param(s_t)) e_{t-1} + \param(s_t) \fvec(s_t) | s_t = s]$, can be expanded as:
\begin{align*}
    e(s) &= \dis(s) \lim_{t\to \infty} \Ep[\gamma (1-\param(s_t)) e_{t-1} + \param(s_t) \fvec_t| s_t = s] \\
    &= \dis(s) \param(s) \fvec(s) + \gamma (1-\param(s)) \dis(s) \lim_{t \to \infty} \Ep[e_{t-1} | s_t = s] \\
    &= \dis(s) \param(s) \fvec(s) + \gamma (1-\param(s)) \dis(s) \lim_{t \to \infty} \sum_{\Bar{s}, \Bar{a}} \mathbb{P}\{s_{t-1} = \Bar{s}, a_{t-1}=\Bar{a} | s_t = s\} \Ep[e_{t-1} | s_{t-1} = \Bar{s}] \\
    &= \dis(s) \param(s) \fvec(s) + \gamma (1-\param(s)) \dis(s) \sum_{\Bar{s}} \frac{\dis(\Bar{s}) \pp(s|\Bar{s})}{\dis(s)} \lim_{t \to \infty} \Ep[e_{t-1} | s_{t-1} = \Bar{s}] \\
    &= \dis(s) \param(s) \fvec(s) + \gamma (1-\param(s)) \sum_{\Bar{s}} \pp(s|\Bar{s}) e(\Bar{s}).
\end{align*}
We can express these quantities in a matrix as,
\begin{align*}
    \keyE^T &= \fmat^T D B + \gamma \keyE^T \pp (I - B), \\
    &= \fmat^T D B + \gamma \fmat^T D B \pp (I - B) + \gamma^2 \fmat^T D B (\pp (I - B)))^2 + \dots , \\
    &= \fmat^T D B (I - \gamma \pp (I - B))^{-1}.
\end{align*}
We can perform a similar analysis on $\keyb$ and substitute $\keyE^T$ to get,
\begin{align}
    \keyA &= \fmat^T D B (I - \gamma \pp (I - B))^{-1} (I - \gamma \pp) \fmat, \\
    \keyb &= \fmat^T D B (I - \gamma \pp (I - B))^{-1} \rp.
\end{align}
These expressions can be simplified further by considering a new transition matrix $\pbeta$ that accounts for the termination due to bootstrapping and discounted by $\gamma$. In other words, $\pbeta$ is made up of $\pp$, but the transitions are terminated according to $B$ and continued according to $(I-B)$ at each step. This is a sub-stochastic matrix for $\gamma \in [0,1)$ and a stochastic matrix when $\gamma=1$. By definition,
\begin{align}
    \pbeta &= \gamma \pp B + \gamma \pp (I-B) \gamma \pp B + (\gamma \pp (I-B))^2 \gamma \pp B + \dots \nonumber \\
    &= \gamma \Big(\sum_{k=0}^{\infty} (\gamma \pp (I-B))^k \Big) \pp B \label{eq:pbeta_1} \\
    &= \gamma (I - \gamma \pp (I-B))^{-1} \pp B \label{eq:pbeta_2} \\
    &= (I - \gamma \pp (I-B))^{-1} (- \gamma \pp (I-B) + \gamma \pp) \nonumber \\
    &= (I - \gamma \pp (I-B))^{-1} (I - \gamma \pp (I-B) + \gamma \pp - I) \nonumber \\
    &= I - (I - \gamma \pp (I-B))^{-1} (I - \gamma \pp) \nonumber \\
    I - \pbeta &= (I - \gamma \pp (I-B))^{-1} (I - \gamma \pp) \label{eq:pbeta_3}.
\end{align}
Therefore $\keyA$ and $\keyb$ are given by,
\begin{align}
\label{Eq:keyA_keyB}
    \keyA &= \fmat^T D B (I - \pbeta) \fmat, \\
    \keyb &= \fmat^T D B (I - \gamma \pp (I - B))^{-1} \rp.
\end{align}
Under assumption \ref{assump:feature_assumption}, $\norm{\fmat} \leq M$, where $M$ is a constant. This implies each quantity in the above expression is bounded and well-defined.
\end{proof}

\begin{theorem*}
\label{app:forward_backward_proof}
The forward and the backward views of PTD are equivalent in expectation: $$\keyb - \keyA \w = \fmat^T D \Big(\Tp (\fmat \w) - \fmat \w \Big).$$
\end{theorem*}
\begin{proof}
We have $\keyA = \fmat^T D B (I - \pbeta) \fmat$ and $\keyb = \fmat^T D B (I - \gamma \pp (I - B))^{-1} \rp$.
\begin{align*}
    \keyb - \keyA \w &= \fmat^T D B (I - \gamma \pp (I - B))^{-1} \rp - \fmat^T D B (I - \pbeta) \fmat \w, \\
    &= \fmat^T D B (I - \gamma \pp (I - B))^{-1} \rp + \fmat^T D B \pbeta \fmat \w - \fmat^T D B \fmat \w, \\
    &= \fmat^T D B (I - \gamma \pp (I - B))^{-1} \rp + \gamma \fmat^T D B (I - \gamma \pp (I - B))^{-1} \pp B \fmat \w - \fmat^T D B \fmat \w, \\
    &= \fmat^T D \Big(B (I - \gamma \pp (I - B))^{-1} (\rp + \gamma \pp B \fmat \w) - B \fmat \w \Big), \\
    &= \fmat^T D \Big(B (I - \gamma \pp (I - B))^{-1} (\rp + \gamma \pp B \fmat \w) + (I - B) \fmat \w  - \fmat \w \Big), \\
    &= \fmat^T D \Big(\Tp(\fmat \w) - \fmat \w \Big).
\end{align*}
We used the definition of $\pbeta$ (c.f equation \ref{eq:pbeta_2}) in the third step.
\end{proof}

\begin{lemma*}
\label{app:col_sums_proof}
The column sums of $\keyA$ are positive.
\end{lemma*}
\begin{proof}
Let $\gamma=1$. The column sums of the key matrix is given by $1^T D B (I - \pbeta) = \dis B - \dis B \pbeta$. We can expand $\dis B \pbeta$ using equation \ref{eq:pbeta_1},
\begin{align*}
   \dis B \pbeta &= \dis B \Big(\sum_{k=0}^{\infty} (\pp (I-B))^k \Big) \pp B, \\
   &= \dis \Big(B + B \sum_{k=1}^{\infty} (\pp (I-B))^k \Big) \pp B, \\
   &= \dis \Big(I - (I-B) + B \sum_{k=1}^{\infty} (\pp (I-B))^k \Big) \pp B, \\
   &= \dis \Big(I + \Big(-I + B \sum_{k=1}^{\infty} \pp ((I-B) \pp)^{k-1} \Big) (I-B) \Big) \pp B, \\
   &= \dis \Big(I + \Big(\underbrace{-I + B \sum_{k=0}^{\infty} \pp ((I-B) \pp)^k}_{=\mathbf{S}} \Big) (I-B) \Big) \pp B, \\
   &= \Big(\dis + \dis \mathbf{S} (I-B) \Big) \pp B \numberthis \label{eq:dis_S}.
\end{align*}
Consider $\dis \mathbf{S}$,
\begin{align*}
    \dis \mathbf{S} &= \dis \Big(-I + B \sum_{k=0}^{\infty} \pp ((I-B) \pp)^k \Big), \\
    &= \dis \Big(-I + B \pp + B \sum_{k=1}^{\infty} \pp ((I-B) \pp)^k \Big), \\
    &= \dis \Big(-I + \pp - (I - B) \pp + B \sum_{k=1}^{\infty} \pp ((I-B) \pp)^k \Big), \\
    &= \underbrace{-\dis + \dis \pp}_{=0,\ \dis \pp = \dis} + \dis \Big(-(I - B) \pp + B \sum_{k=1}^{\infty} \pp ((I-B) \pp)^k \Big), \\
    &= \dis \Big(-I + B \sum_{k=1}^{\infty} \pp ((I-B) \pp)^{k-1} \Big) (I-B) \pp, \\
    &= \dis \Big(-I + B \sum_{k=0}^{\infty} \pp ((I-B) \pp)^k \Big) (I-B) \pp, \\
    &= \dis \mathbf{S} (I-B) \pp.
\end{align*}
We can recursively expand $\dis \mathbf{S}$ $n$ times to get $\dis \mathbf{S} = \dis \mathbf{S} ((I-B) \pp)^n$. Notice that $(I-B) \pp$ is a sub-stochastic matrix, therefore, the elements of the matrix keep getting smaller as $n$ gets bigger. In fact, $\lim_{n \to \infty} ((I-B) \pp)^n = 0$. Therefore, $\dis \mathbf{S} = 0$ as $n \to \infty$. We can use this result in equation \ref{eq:dis_S} to get,
\begin{equation}
   \dis B \pbeta = \dis \pp B \ = \dis B.
\end{equation}
This implies, $\dis B (I - \pbeta) = 0$ for $\gamma=1$ and $\dis B (I - \pbeta) > 0$ for $\gamma \in [0, 1)$ as $\dis B \pbeta < \dis B$. Therefore, column sums are positive.
\end{proof}

\begin{lemma*}
\label{app:sampling_bias_proof}
We have, $\sum_{t=0}^{\infty} \norm{\Ep[A(X_t)|X_0] - \keyA} \leq C_1$ and $\sum_{t=0}^{\infty} \norm{\Ep[b(X_t)|X_0] - \keyb} \leq C_2$ under assumption \ref{assump:rapid_mixing}, where $C_1$ and $C_2$ are constants.
\end{lemma*}
\begin{proof}
By similar analysis to finding $\keyA$ (lemma \ref{remark:keyA_keyB_definition}), we have $$\Ep[A(X_t)|X_0] = \fmat^T B D_t (I - \pbeta) \fmat - \fmat^T B D_t \sum_{m=t+1}^{\infty} (\gamma \pp (I-B))^m (I -\gamma \pp) \fmat,$$ where $D_t(i,i) = \prob(s_t = s_i | s_0)$. Now,
\begin{align*}
    &\norm{\Ep[A(X_t)|X_0] - \keyA} = \norm{\fmat^T B (D_t - D) (I - \pbeta) \fmat - \fmat^T B D_t \sum_{m=t+1}^{\infty} (\gamma \pp (I-B))^m (I -\gamma \pp) \fmat}, \\
    &\leq \norm{\fmat^T B (D_t - D) (I - \pbeta) \fmat} + \norm{\fmat^T B D_t \sum_{m=t+1}^{\infty} (\gamma \pp (I-B))^m (I -\gamma \pp) \fmat}, \\
    &\leq \underbrace{\norm{\fmat^T} \norm{B}}_{\text{constant}} \norm{D_t - D} \underbrace{\norm{I - \pbeta} \norm{\fmat}}_{\text{constant}} + \underbrace{\norm{\fmat^T} \norm{B} \norm{D_t}}_{\text{constant}} \norm{\sum_{m=t+1}^{\infty} (\gamma \pp (I-B))^m} \underbrace{\norm{(I -\gamma \pp)} \norm{\fmat}}_{\text{constant}}, \\
    &\leq  K_1 C \rho^t + K_2 \sum_{m=t+1}^{\infty} \gamma^m \underbrace{\norm{(\pp (I-B))^m}}_{\leq K_5, \forall m}, \\
    &\leq K_1 C \rho^t + K_3 \sum_{m=t+1}^{\infty} \gamma^m, \\
    &= K_4 \rho^t + \frac{\gamma^{t+1}}{1-\gamma} K_3,
\end{align*}
where $K_1, K_2, K_3, K_4, K_5$ are constants, we have used the triangle inequality and $\norm{AB} \leq \norm{A} \cdot \norm{B}$ properties. Now,
\begin{align*}
    \sum_{t=0}^{\infty}\norm{\Ep[A(X_t)|X_0] - \keyA} &\leq \sum_{t=0}^{\infty} K_3 \rho^t + \frac{\gamma^{t+1}}{1-\gamma} K_2, \\
    &= K_3 \frac{1}{1-\rho} + K_2 \frac{\gamma}{(1-\gamma)^2}, \\
    &= C_1.
\end{align*}
By a similar analysis, we get $\sum_{t=0}^{\infty} \norm{\Ep[b(X_t)|X_0] - \keyb} \leq C_2$. Therefore, the bias due to sampling is contained.
\end{proof}

\subsection{Standard results used in the proof}
\label{app_subsec:std_results}

\begin{corollary*}
\label{app:std_corollary}
\citep{varga1999matrix} If $\keyA$ is a Hermitian $n \times n$ strictly diagonally dominant or irreducibly diagonally dominant matrix with positive real diagonal entries, then $\keyA$ is positive definite.
\end{corollary*}

\begin{theorem*}
\label{app:std_theorem}
\citep{bertsekas1996neuro, tsitsiklis1997analysis} Consider an iterative algorithm of the form $\w_{t+1} = \w_t + \alpha_t (A(X_t) \w_t + b(X_t))$ where,
\begin{enumerate}
    \item the step-size sequence $\alpha_t$ satisfies assumption \ref{assump:step_size};
    \item $X_t$ is a Markov process with a unique invariant distribution, and there exists a mapping $h$ from the states of the Markov process to the positive reals, satisfying the remaining conditions;
    \item $A(\cdot)$ and $b(\cdot)$ are matrix and vector valued functions respectively, for which $\keyA = \Edis[A(X_t)]$ and $\keyb = \Edis[b(X_t)]$ are well defined and finite;
    \item the matrix $\keyA$ is negative definite;
    \item there exists constants $C$ and $q$ such that for all $X$
    \begin{itemize}
        \item $\sum_{t=0}^{\infty} \norm{\Ep[A(X_t) | X_0 = X] - \keyA} \leq C (1 + h^q(X))$, and
        \item $\sum_{t=0}^{\infty} \norm{\Ep[b(X_t) | X_0 = X] - \keyb} \leq C (1 + h^q(X))$;
    \end{itemize}
    \item for any $q > 1$ there exists a constant $\mu_q$ such that for all $X, t$
    \begin{itemize}
        \item $\Ep[h^q(X_t)|X_0=X] \leq \mu_q (1 + h^q(X)).$
    \end{itemize}
\end{enumerate}
Then, $\w_t$ converges to $\wpi$, with probability one, where $\wpi$ is the unique vector that satisfies $\keyA \wpi + \keyb = 0$.
\end{theorem*}

\subsection{Counterexamples of TD($\lambda$)}
\label{app_sec:counterexamples}


\textbf{Example 2 \citep{ghiassian2017first}:} Two state MDP with $\pp = \begin{bmatrix} 0 & 1\\ 1 & 0 \end{bmatrix}$, $\fmat = \begin{bmatrix} 3 & 1 \\1 & 1 \end{bmatrix}$, $\lambda = \begin{bmatrix} 0 & 0.99 \end{bmatrix}$, $\dis = \begin{bmatrix} 0.5 & 0.5 \end{bmatrix}$, $\gamma = 0.95$. The key matrix of TD($\lambda$) is given by, $\keyA = \begin{bmatrix} -0.46 & 0.15 \\ -0.77 & 0.07  \end{bmatrix}$, which is not positive definite. Therefore, the updates will not converge. The key matrix of Preferential TD for the same setup is $\keyA =  \begin{bmatrix} 0.46 & 0.15 \\ 0.15 & 0.05 \end{bmatrix}$ which is positive definite.

\subsection{Experiments}

\subsubsection{Tabular setting}
\label{app_subsec:tabular_experiments}

\begin{figure}[ht]
    \centering
    \includegraphics[width=0.85\textwidth, height=4.5cm]{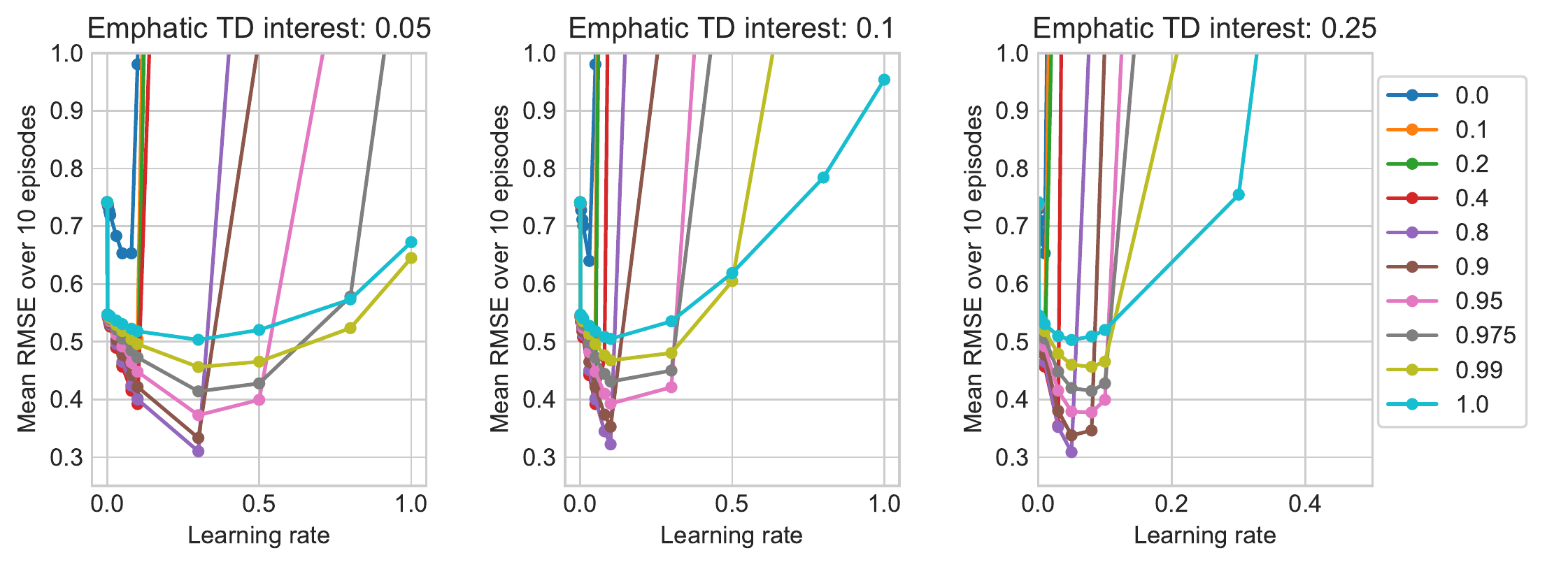}
    \caption{Performance of Emphatic TD with a fixed interest value (shown in the title of the plot). Average RMSE of the first 10 episodes is plotted against learning rate for different choices of $\lambda$.}
    \label{app_fig:tabular_hp_plot2}
\end{figure}

We analyzed the bias-variance trade-off of Emphatic TD with several choices of fixed interest values. ETD with a low interest value resulted in lower errors compared to high interest values. The difference in performance is due to the algorithm's sensitivity to learning rates when high interest is used. We present the root mean squared error obtained for various fixed interests in Figure \ref{app_fig:tabular_hp_plot2}. Different curves in the plot correspond to various values of $\lambda$. We generated the plots by averaging the root mean squared error of the value function over 10 episodes of training across 25 seeds.

\subsubsection{Linear setting}
\label{app_subsec:linear_experiments}

The error curves for the corridor lengths \{10, 20\} are provided in Figure \ref{app_fig:linear_learning}. The observations made in the main paper holds true for these lengths as well.

\begin{figure}[H]
    \centering
    \begin{subfigure}{0.4\textwidth}
        \centering
        \includegraphics[width=\textwidth, height=3.8cm]{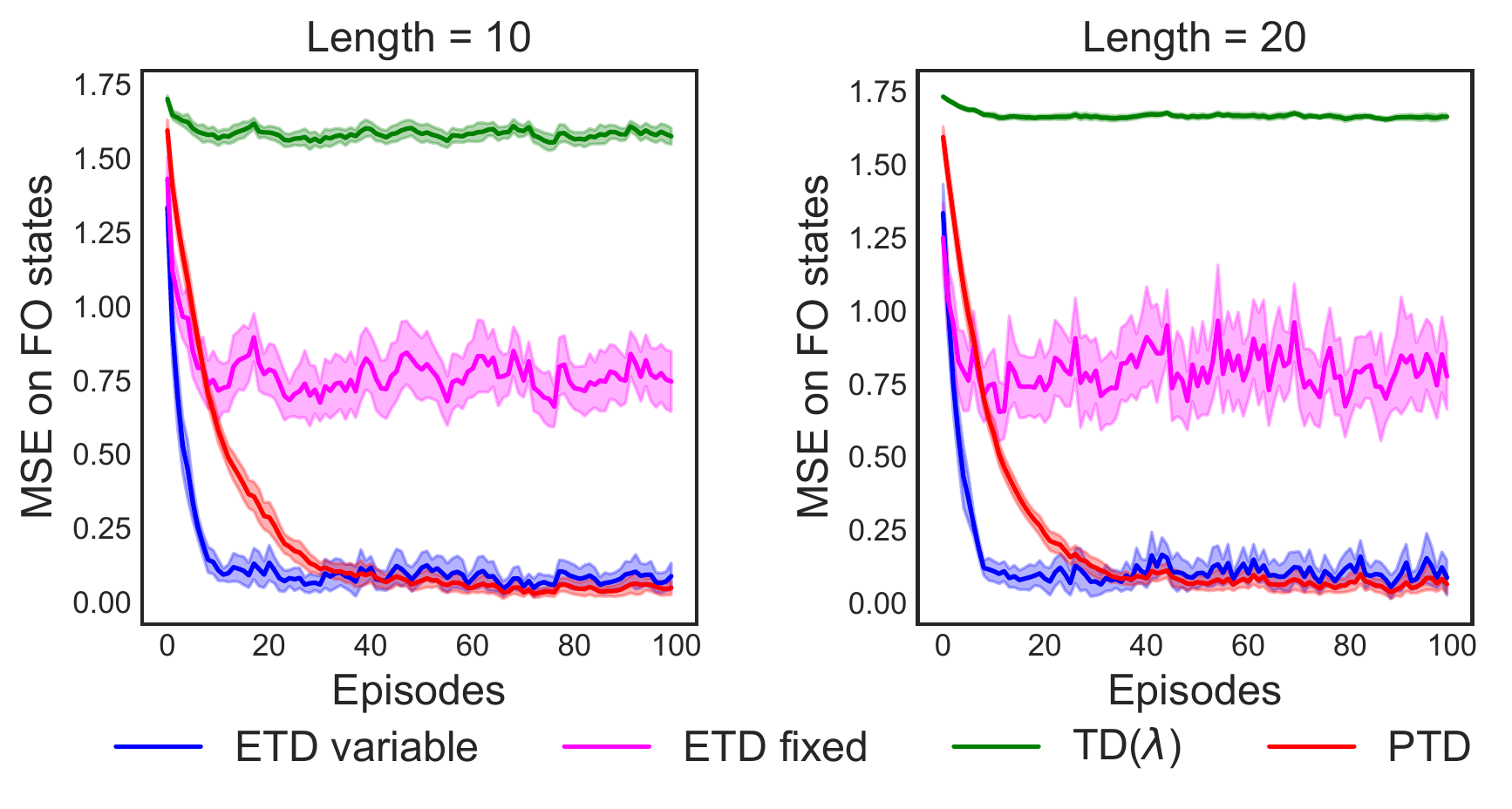}
    \end{subfigure}
    \begin{subfigure}{0.4\textwidth}
        \centering
        \includegraphics[width=\textwidth, height=3.8cm]{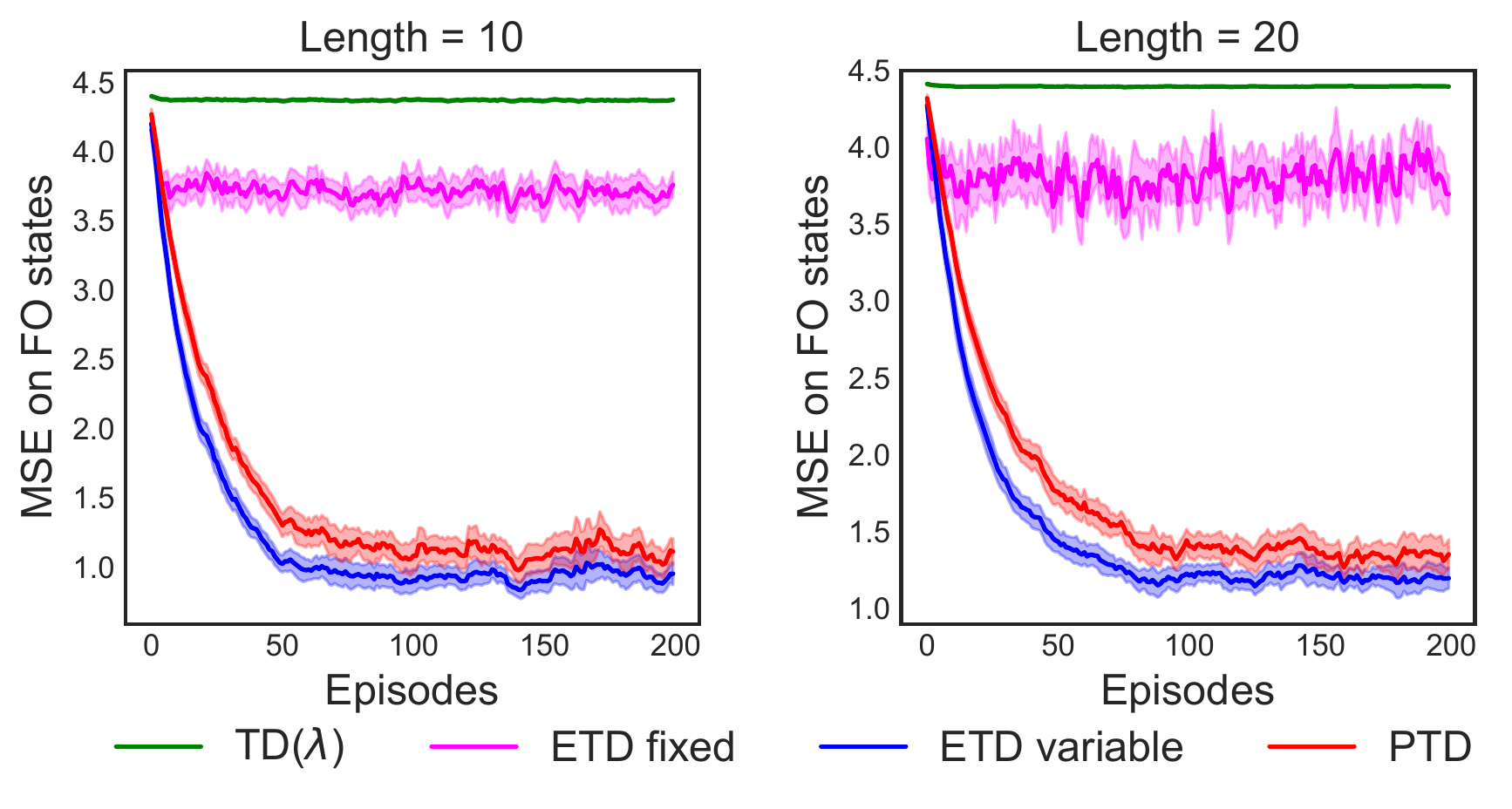}
    \end{subfigure}
    \caption{The mean squared error of fully observable states' values is plotted as a function of episodes for various algorithms. The first two plots correspond to the results on Task 1 (left) and the next plots correspond to Task 2 (right). The corridor length is indicated in the title of the plot.}
    \label{app_fig:linear_learning}
\end{figure}

\begin{figure}[H]
    \centering
    \begin{subfigure}{\textwidth}
        \centering
        \includegraphics[width=\textwidth, height=3.8cm]{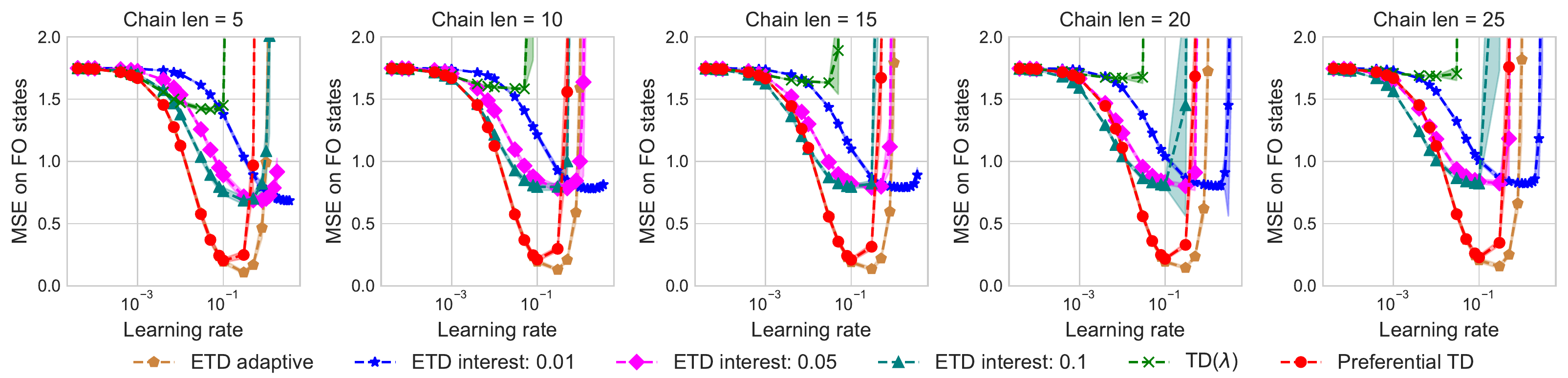}
    \end{subfigure}
    \hfill
    \begin{subfigure}{\textwidth}
        \centering
        \includegraphics[width=\textwidth, height=3.8cm]{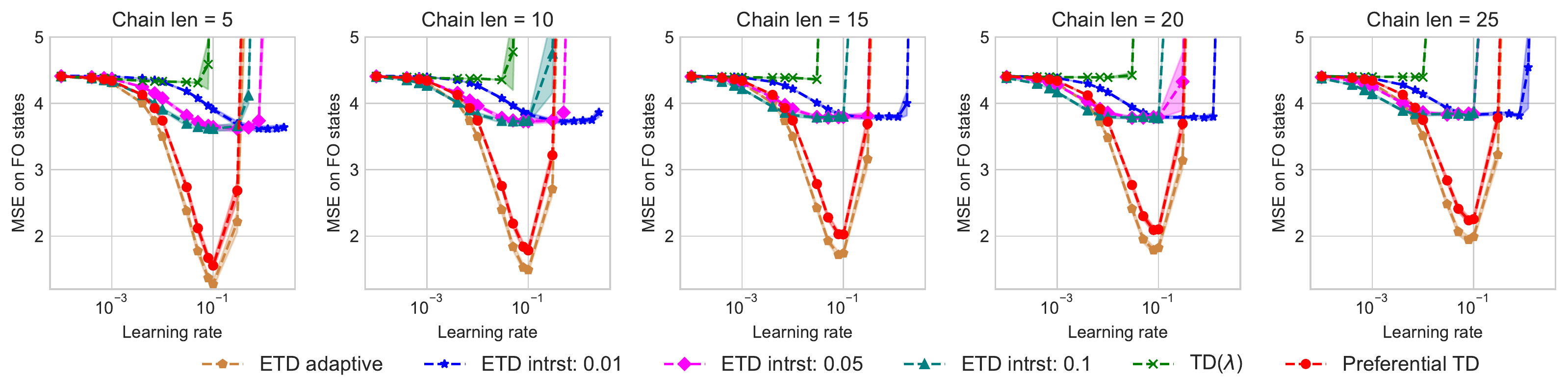}
    \end{subfigure}
    \caption{The average mean squared error of the fully observable states' values is plotted against learning rate on Task 1 (top) and Task 2 (bottom).}
    \label{app_fig:linear_hptuning}
\end{figure}

\begin{figure}[H]
    \centering
    \begin{subfigure}{\textwidth}
        \centering
        \includegraphics[width=\textwidth, height=3.8cm]{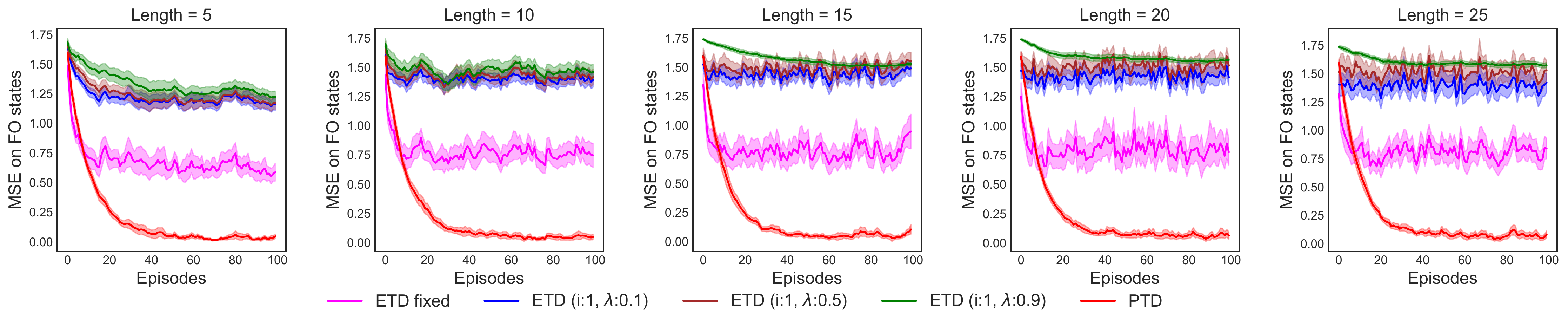}
    \end{subfigure}
    \hfill
    \begin{subfigure}{\textwidth}
        \centering
        \includegraphics[width=\textwidth, height=3.8cm]{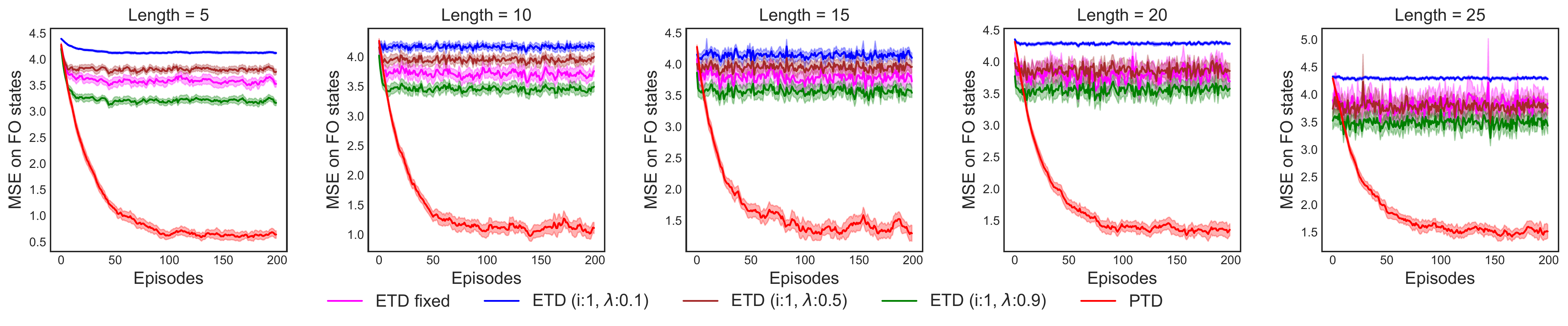}
    \end{subfigure}
    \caption{The mean squared error of the fully observable states' values is plotted against episodes on Task 1 (top) and Task 2 (bottom).}
    \label{app_fig:ETD_fixed_v2}
\end{figure}

\begin{figure}[ht]
    \centering
    \begin{subfigure}{\textwidth}
        \centering
        \includegraphics[width=\textwidth, height=3.8cm]{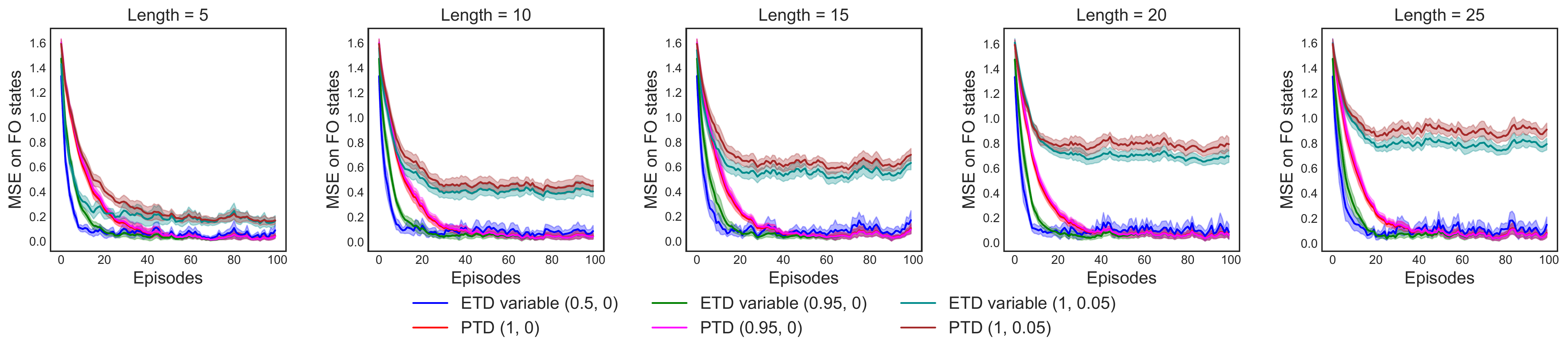}
    \end{subfigure}
    \hfill
    \begin{subfigure}{\textwidth}
        \centering
        \includegraphics[width=\textwidth, height=3.8cm]{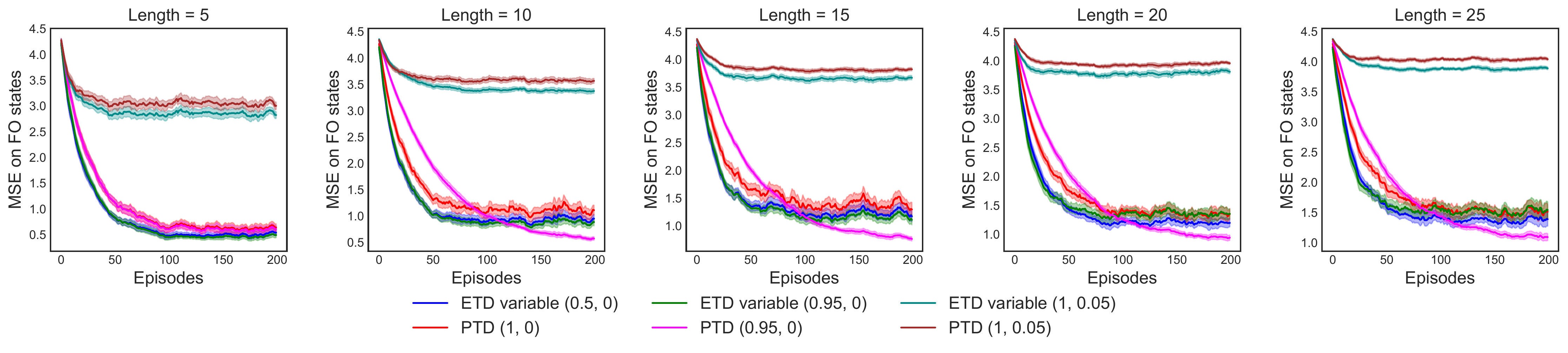}
    \end{subfigure}
    \caption{Performance plot of PTD and ETD-variable for different choices of $\beta$ (or $i$ and $\lambda$) for FO and PO states on Task 1 (top) and Task 2 (bottom).}
    \label{app_fig:PTD_ETD_v2}
\end{figure}

\textbf{Task 1:} We chose the best learning rate from \{1.2, 1.0, 0.8, 0.5, 0.3, 0.1, 8e-2, 5e-2, 3e-2, 1e-2, 7e-3, 4e-3, 1e-3, 7e-4, 4e-4, 1e-4, 7e-5, 4e-5\}. We extended the range of search to include \{10.0, 5.0, 4.0, 3.5, 3.0, 2.5, 2.0, 1.8, 1.5,\} for ETD-fixed. We also ran ETD-fixed on 3 different interests (\{0.01, 0.05, 0.1\}). We ran all the algorithms for 100 episodes. We calculated the mean squared error of fully observable states' values and averaged it across 100 episodes and 25 seeds for the above mentioned learning rates. We chose the learning rate that resulted in the lowest error. The hyperparameter tuning plot is presented in Figure \ref{app_fig:linear_hptuning}. We ran all the algorithms on 25 different seeds and plotted the mean error with a confidence interval of 0.5 times the standard deviation.

\textbf{Task 2:} The learning rate for this task was selected from the set \{0.8, 0.5, 0.3, 0.1, 8e-2, 5e-2, 3e-2, 1e-2, 7e-3, 4e-3, 1e-3, 7e-4, 4e-4, 1e-4\}. We extended the range of search to include \{2.5, 1.8, 1.2,\} for ETD-fixed. We also ran ETD-fixed on 3 different interest values (\{0.01, 0.05, 0.1\}). We calculated the mean squared error of fully observable states' values and averaged it across 100 episodes and 25 seeds for the above mentioned learning rates. We choose the learning rate that resulted in the lowest error. The hyperparameter tuning plot is presented in Figure \ref{app_fig:linear_hptuning}. We ran all the algorithms on 25 different seeds and plotted the mean with a confidence interval of 0.5 times the standard deviation. We ran 200 episodes to generate the learning curves presented in Figure \ref{fig:linear_learning} and \ref{app_fig:linear_learning} using the optimal learning rate.

\textbf{ETD-fixed with $i=1$ and a constant $\lambda$ for all the states:} In this experiment, we tested the performance of a different version of ETD-fixed. In this version, we set the interest and $\lambda$ to constant values for all the states. The interest value was set to $1$, and we experimented with three separate values of $\lambda$ - \{0.1, 0.5, 0.9\}. Like the earlier version, this too performs poorly compared to PTD and ETD-variable. This is because, setting $i=1$ results in updating all the states including those that are partially observable causing poor generalization. The results are presented in Figure \ref{app_fig:ETD_fixed_v2}.

\textbf{PTD and ETD-variable with non-extreme values:} In this experiment, we tested PTD and ETD-variable with non-extreme preference and (interest, $\lambda$) values respectively. The performance is not affected when $\param$ (or $i$) of fully observable states is set to a value $<1$ as the updates on partially observable states is still blocked. However, the performance drops significantly when $\param$ (or $i$) of partially observable states is set to non-zero values. The results are presented in Figure \ref{app_fig:PTD_ETD_v2}.

\subsubsection{Semi-linear setting}
\label{app_subsec:semi-linear}

\begin{figure}[ht]
    \centering
    \includegraphics[width=0.7\textwidth, height=4cm]{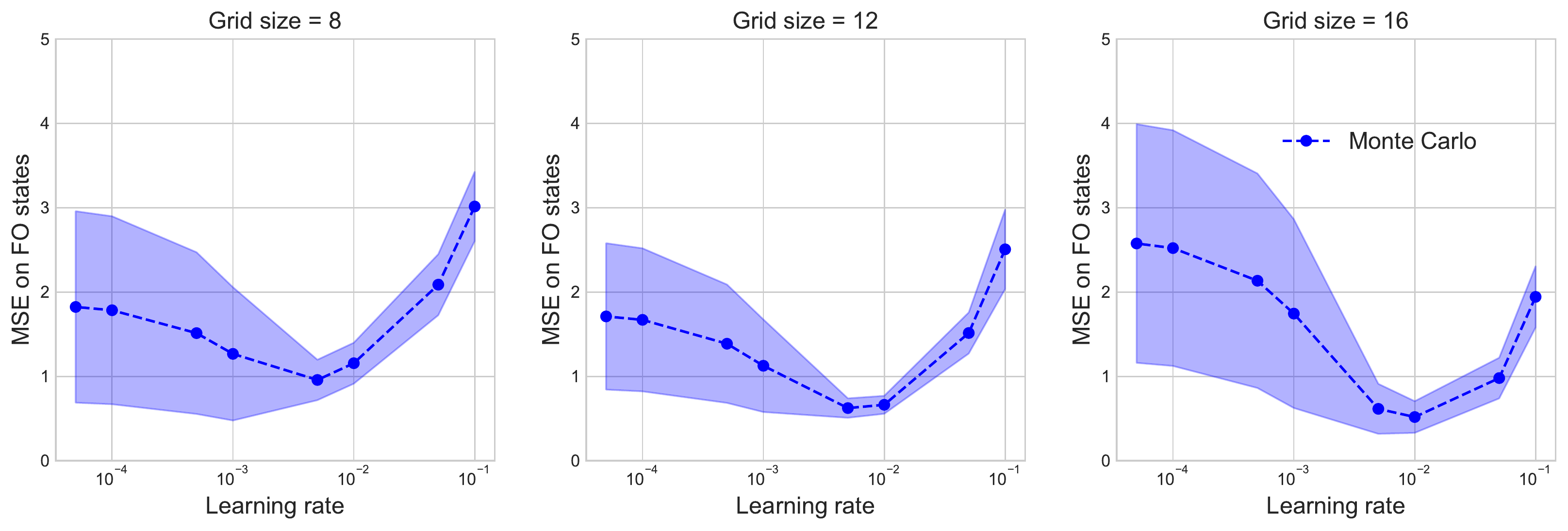}
    \caption{Average mean squared error of the value function is plotted against learning rate for the feature net. Different plots correspond to various grid sizes.}
    \label{app_fig:semilinear_MC_hp}
\end{figure}

\textbf{Feature net:} Feature net is a single-layered neural network with 32 neurons with ReLU non-linearity in the hidden layer. The network's input is a one-hot vector of size $n \times n$, where $n$ is the size of the grid. The component corresponding to the agent's location is set to $1$ and the remaining bits are set to $0$. The network is trained to minimize the mean squared error of Monte Carlo returns and its predictions for the states visited in a trajectory. We train the network using ADAM optimizer. We find the optimal learning rate from \{1e-1, 5e-2, 1e-2, 5e-3, 1e-3, 5e-4, 1e-4, 5e-5\} for every grid size through hyperparameter search. The networks are trained on 50 episodes on each grid size and the mean squared error of the value function across 50 episodes is used as a metric to pick the optimal learning rate. The hyperparameter tuning results are presented in Figure \ref{app_fig:semilinear_MC_hp}. We run experiments on 25 seeds and use a confidence interval of 0.5 times the standard deviation for plotting.

\textbf{Linear:} A linear function approximator is used to estimate the value function of the fully observable states. The hidden layer output of a fully trained feature net is used as features for the input state. A one-hot vector (for fully observable states) or a Gaussian vector (each component is generated from $\mathcal{N}(0,1)$ for partially observable states) is used as an input to the feature net to get features for the downstream task.

\begin{figure}[ht]
    \centering
    \begin{subfigure}{\textwidth}
        \centering
        \includegraphics[width=0.7\textwidth, height=3.5cm]{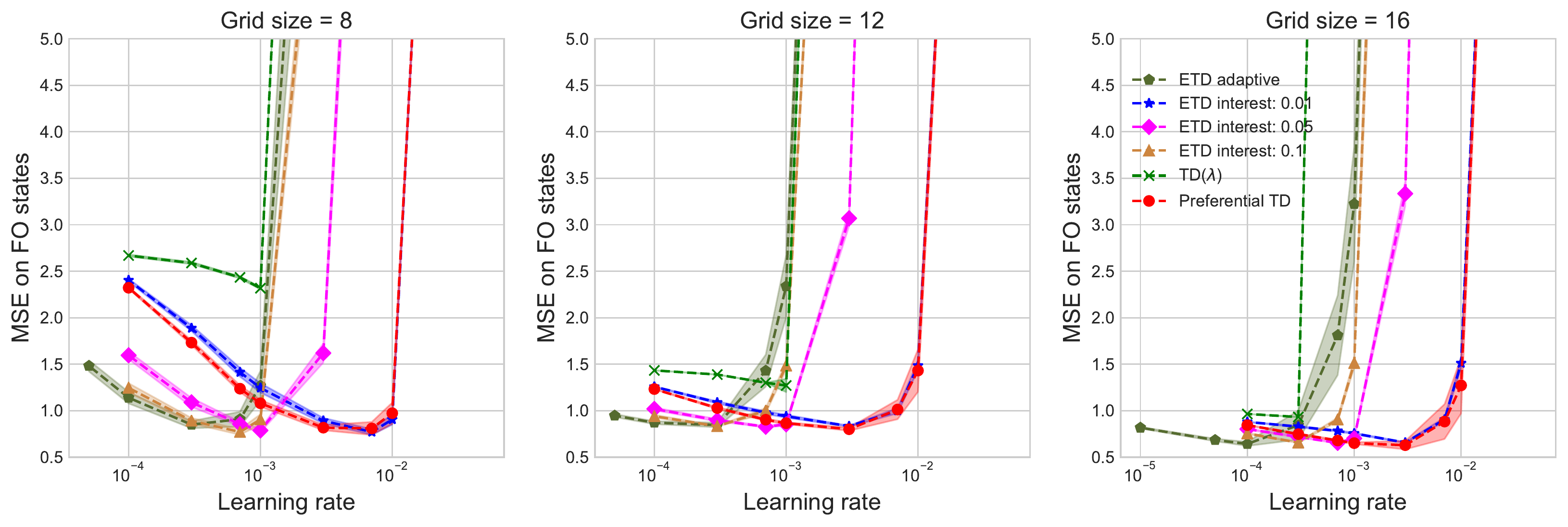}
    \end{subfigure}
    \hfill
    \begin{subfigure}{\textwidth}
        \centering
        \includegraphics[width=0.7\textwidth, height=3.5cm]{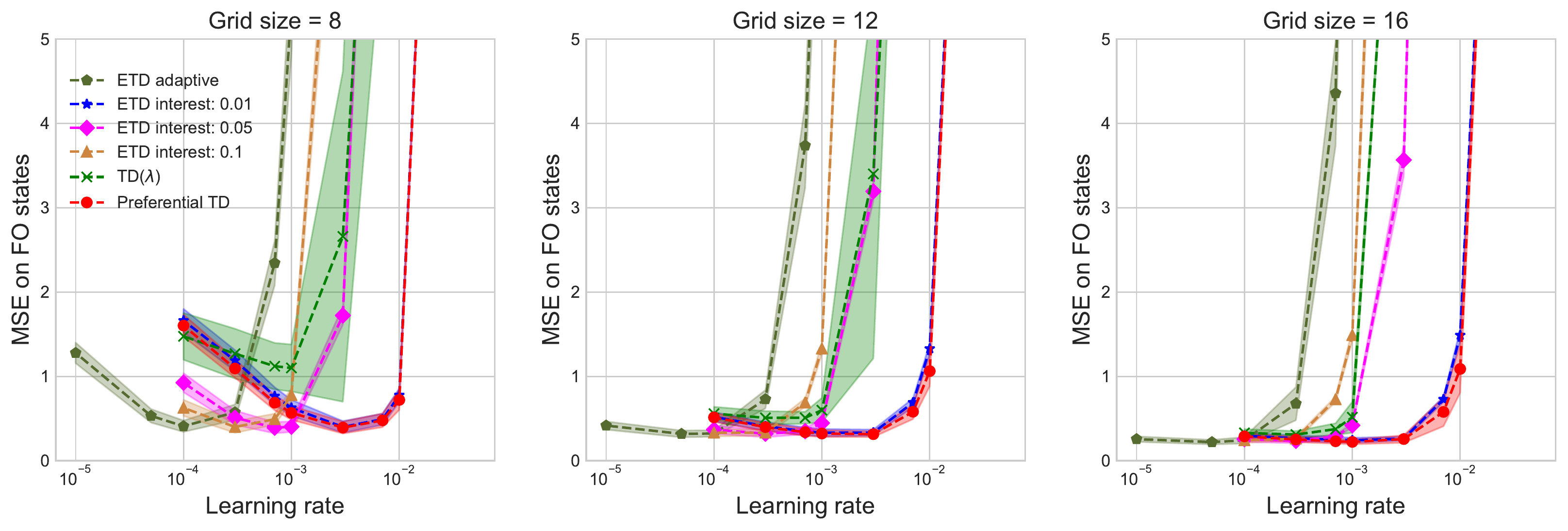}
    \end{subfigure}
    \caption{The average mean squared error of fully observable states' values plotted against various learning rates for Task 1 (top) and Task 2 (bottom).}
    \label{app_fig:semilinear_DRL_hp}
\end{figure}

\begin{figure}[H]
    \centering
    \begin{subfigure}{0.48\textwidth}
        \centering
        \includegraphics[width=\textwidth, height=3.8cm]{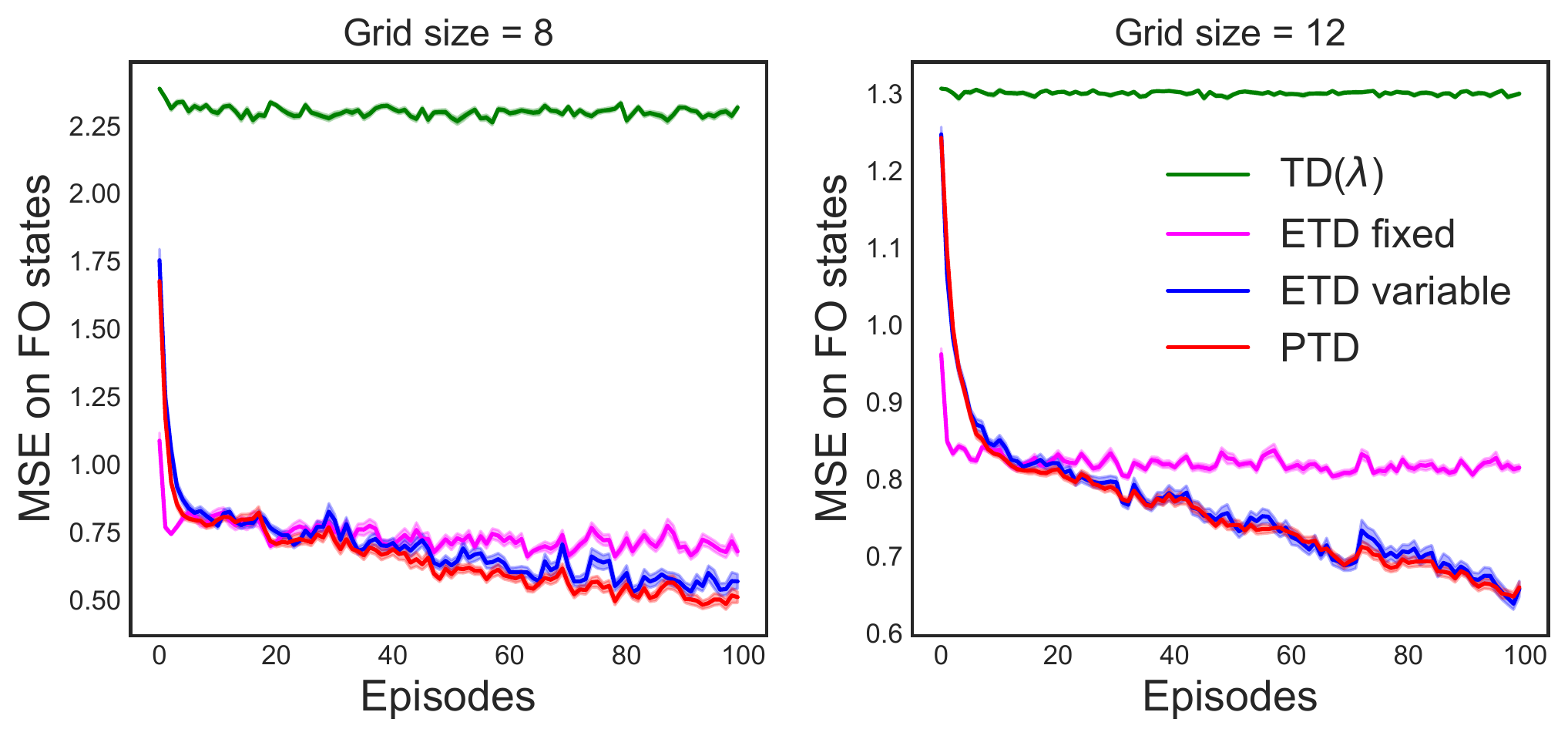}
        \caption{Learning curves on Task 1}
        \label{app_fig:task1_learning_semi}
    \end{subfigure}
    \begin{subfigure}{0.48\textwidth}
        \centering
        \includegraphics[width=\textwidth, height=3.8cm]{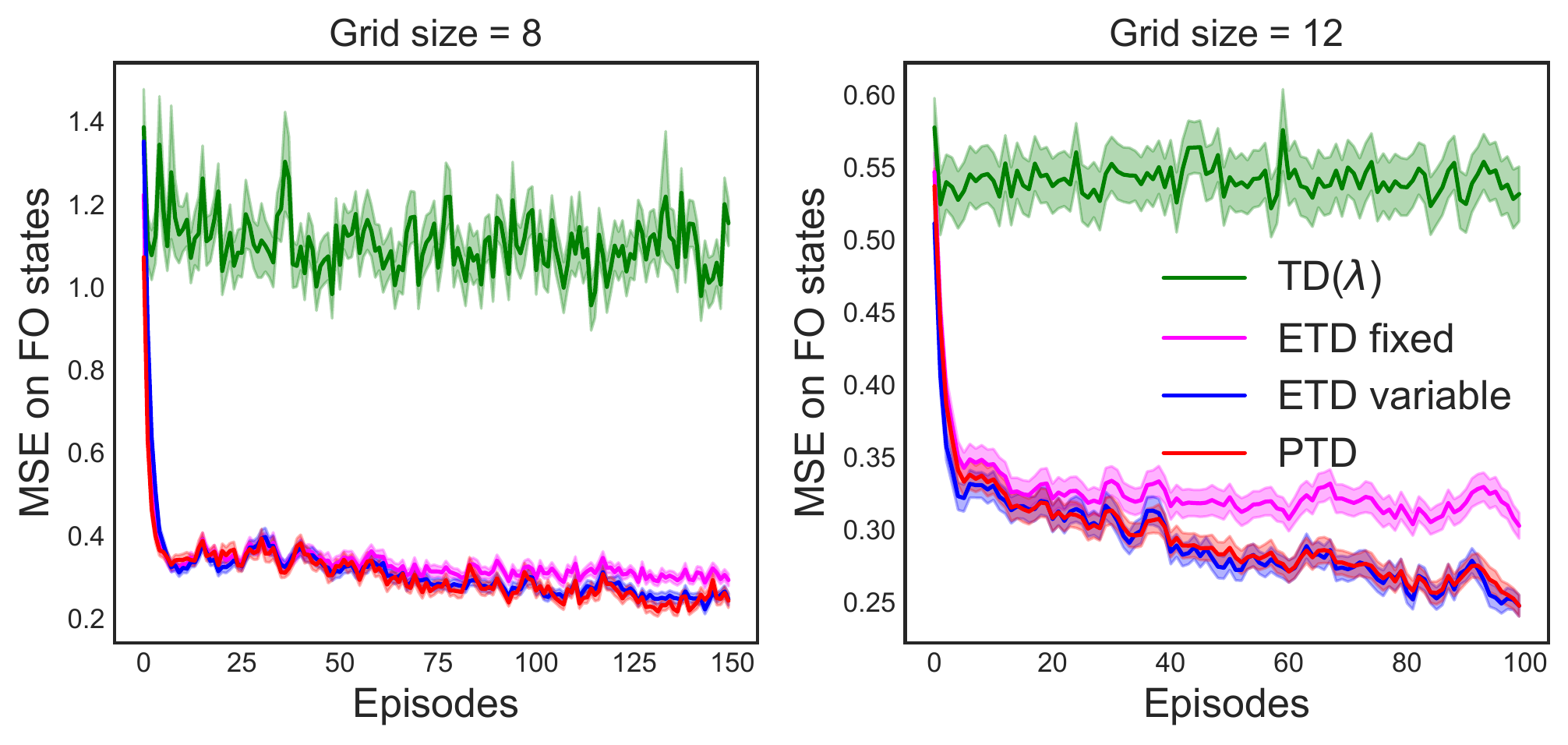}
        \caption{Learning curves on Task 2}
        \label{app_fig:task2_learning_semi}
    \end{subfigure}
    \caption{The mean squared error of fully observable states' values plotted against episodes for various algorithms. The grid size is indicated in the title of the plot.}
    \label{app_fig:semi_learning}
\end{figure}

\textbf{Task 1:} We chose the optimal learning rate from \{5e-2, 1e-2, 7e-3, 3e-3, 1e-3, 7e-4, 3e-4, 1e-4\}. We extended the range of search to include \{5e-5, 1e-5\} for ETD-variable. We also ran ETD-fixed on 3 different interest values - (\{0.01, 0.05, 0.1\}). We ran all the algorithms for 50 episodes. We calculated the mean squared error of fully observable states' values over 50 episodes for the learning rates mentioned earlier and chose the best learning rate based on the average MSE across episodes and seeds. The hyperparameter tuning plot is presented in Figure \ref{app_fig:semilinear_DRL_hp}. We ran all the algorithms on 25 different seeds and plotted the mean MSE with a confidence interval of 0.5 times the standard deviation. We ran 100 episodes for $8 \times 8$ and $12 \times 12$ grids, 50 episodes for $16 \times 16$ grid to generate the learning curves presented in the main paper using the optimal learning rates.

\textbf{Task 2:} All the experimental details are same as task 1. The hyperparameter tuning results are presented in Figure \ref{app_fig:semilinear_DRL_hp}. Once the best learning rate was found for all the algorithms, we ran 150 episodes for $8 \times 8$ grid, 100 episodes for $12 \times 12$ and $16 \times 16$ grids to generate the learning curves presented in Figure \ref{fig:semilinear_learning} and \ref{app_fig:semi_learning}.

\textbf{Additional results:} The error curves for the grid sizes 8 and 12 are presented in Figure \ref{app_fig:semi_learning}. The observations made in the main paper hold true for these grids too. However, both ETD-variable and PTD perform similarly. The other two methods result in poor predictions.

\subsubsection{Non-Linear setting}
\label{app_subsec:non_linear}

\textbf{Forward view:} For each algorithm and network combination, we find the best learning rate from \{5e-2, 1e-2, 5e-3, 1e-3, 5e-4\}. We included 1e-4 in the search space for grid task 2. The best learning rate was decided based on the area under the error curve (AUC) of MSE averaged across episodes and seeds. We used SGD optimizer to train the network. We consider 250 episodes and 25 seeds for all the tasks for hyperparameter tuning and use 50 seeds to plot the final results. The hyperparameter tuning results are presented in Figures \ref{app_fig:non_linear_fwd_grid1} and \ref{app_fig:non_linear_fwd_grid2}.

\textbf{Backward view:} For each algorithm and network combination, we find the best learning rate from \{0.1, 5e-2, 1e-2, 5e-3, 1e-3, 5e-4, 1e-4\}. The best learning rate was decided using the process described in \textit{Forward view}. Plotting details also remain the same as before. The hyperparameter tuning results are presented in Figures \ref{app_fig:non_linear_bwd_grid1} and \ref{app_fig:non_linear_bwd_grid2}.

\begin{figure}[H]
    \centering
    \includegraphics[width=\textwidth, height=10cm]{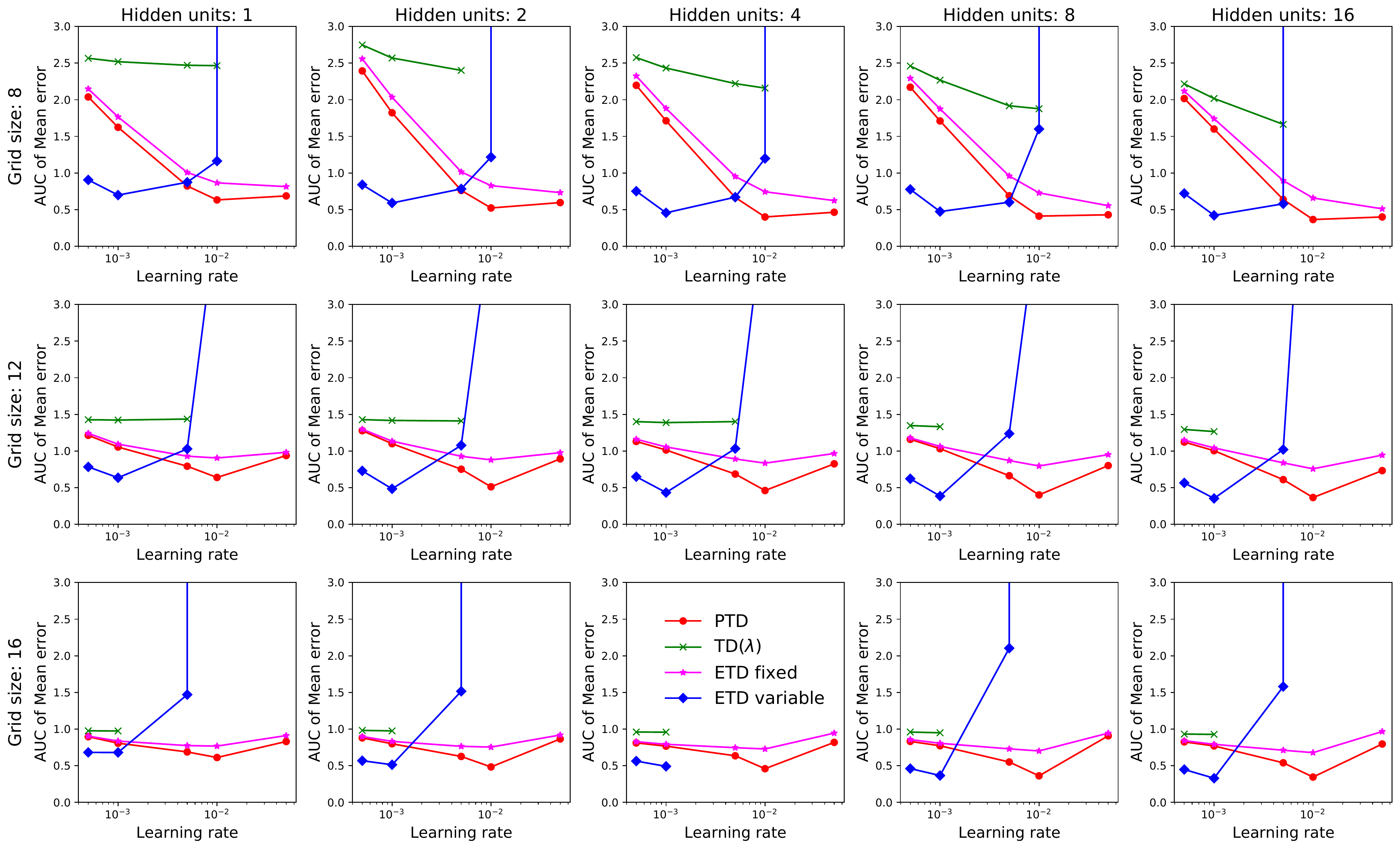}
    \caption{Task 1 forward view hyperparameter tuning curves. AUC of mean error is plotted against learning rates.}
    \label{app_fig:non_linear_fwd_grid1}
\end{figure}

\begin{figure}[H]
    \centering
    \includegraphics[width=\textwidth, height=10cm]{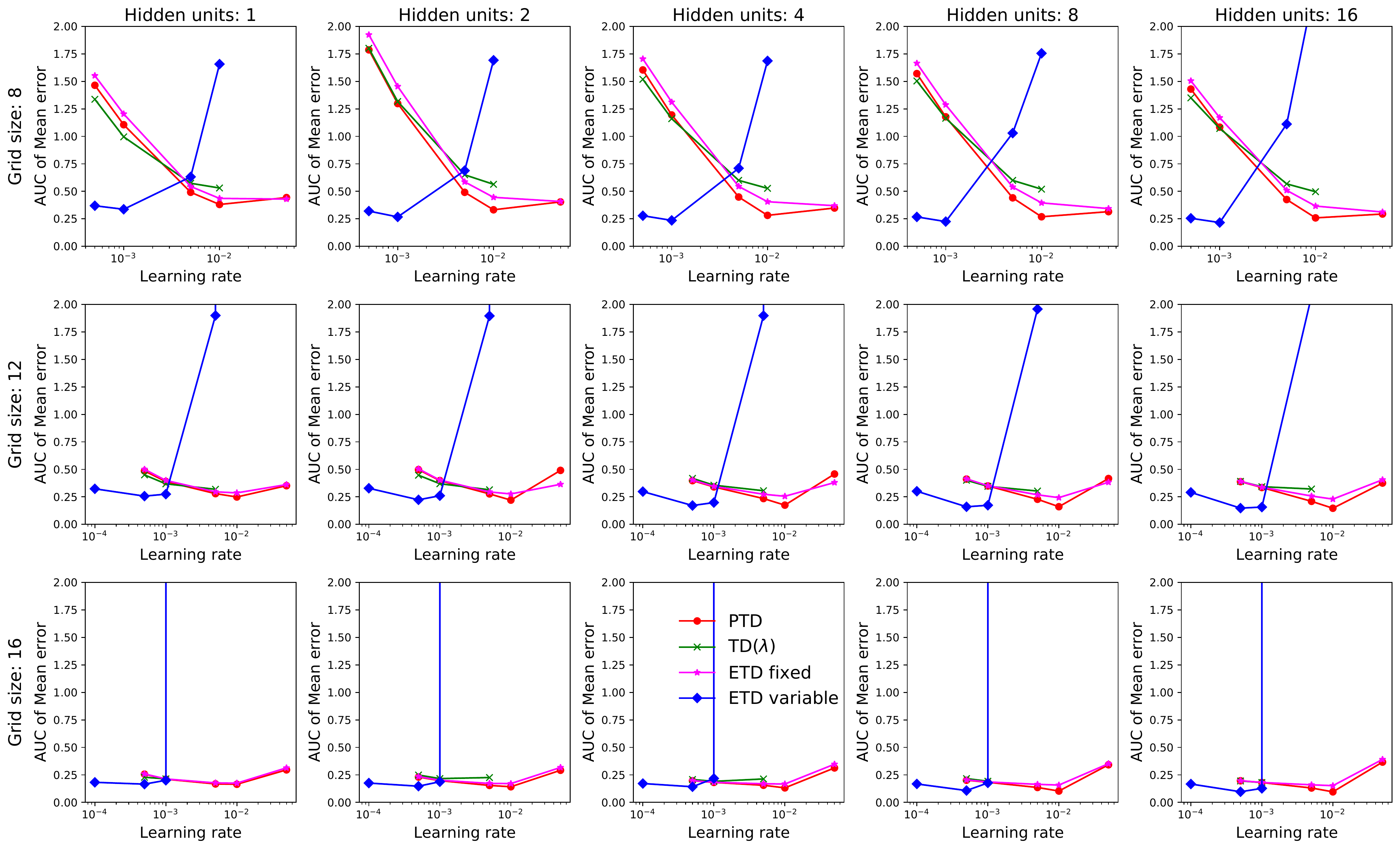}
    \caption{Task 2 forward view hyperparameter tuning curves. AUC of mean error is plotted against learning rates.}
    \label{app_fig:non_linear_fwd_grid2}
\end{figure}

\begin{figure}[H]
    \centering
    \includegraphics[width=\textwidth, height=10cm]{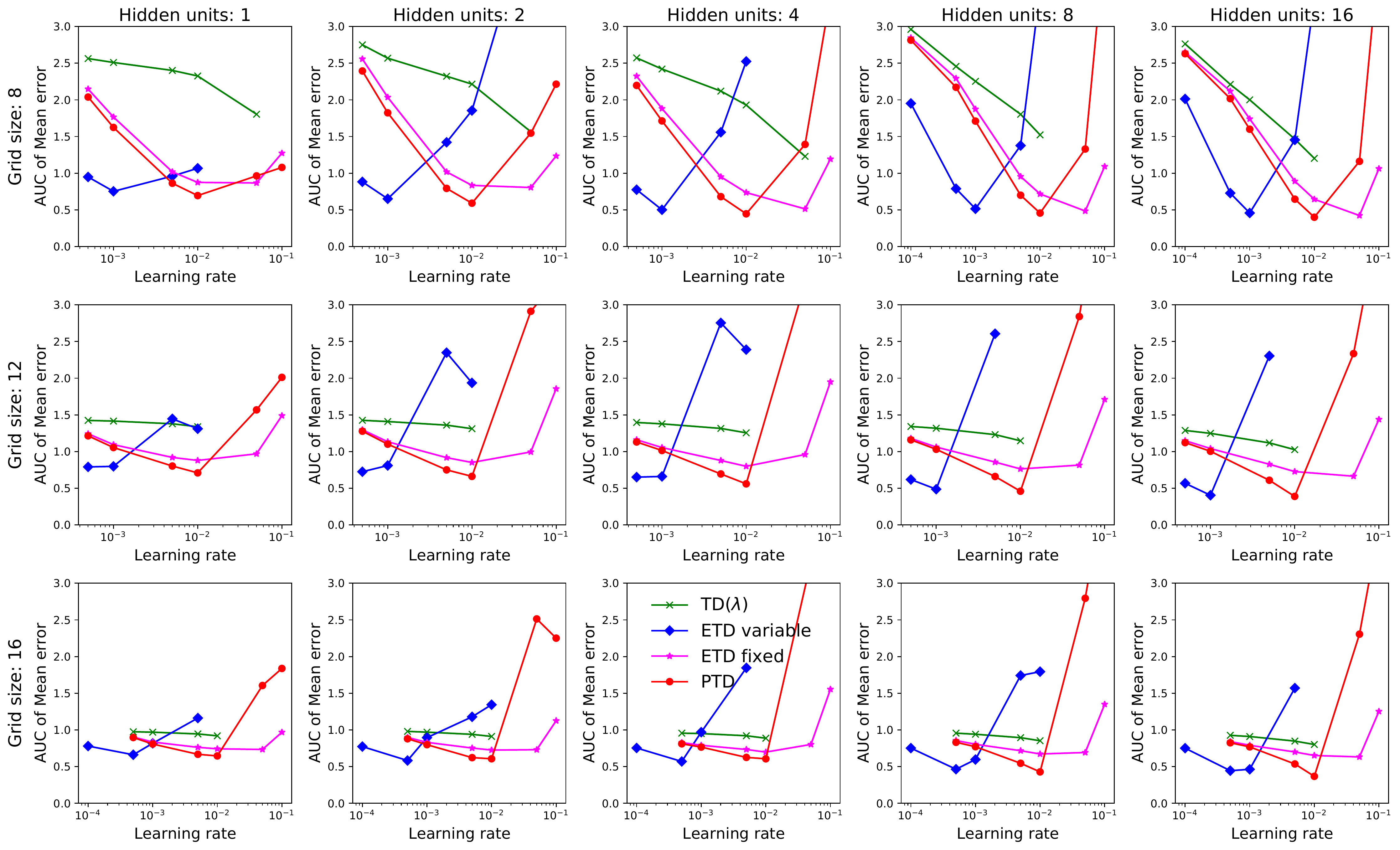}
    \caption{Task 1 backward view hyperparameter tuning curves. AUC of mean error is plotted against learning rates.}
    \label{app_fig:non_linear_bwd_grid1}
\end{figure}

\begin{figure}[H]
    \centering
    \includegraphics[width=\textwidth, height=10cm]{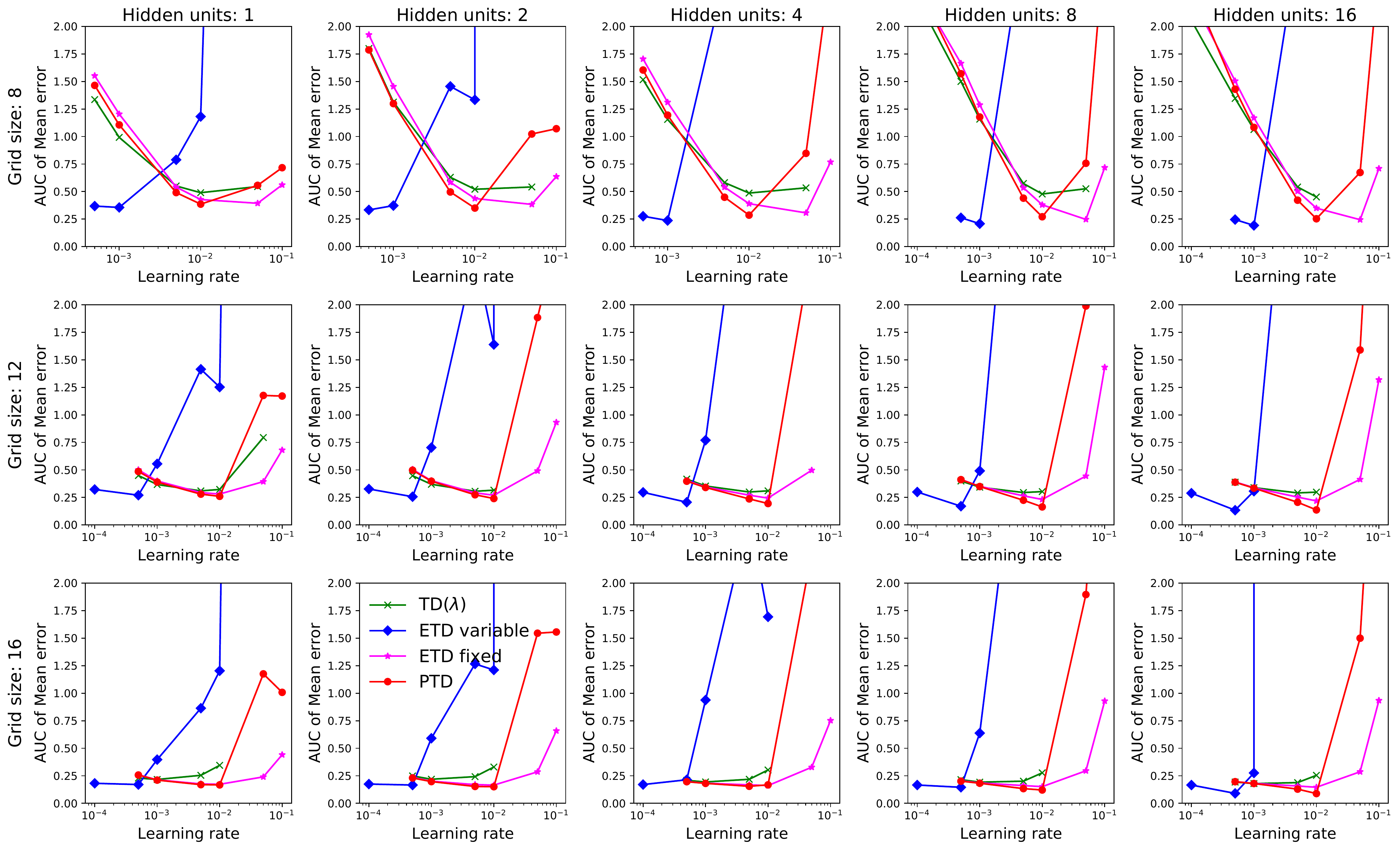}
    \caption{Task 2 backward view hyperparameter tuning curves. AUC of mean error is plotted against learning rates.}
    \label{app_fig:non_linear_bwd_grid2}
\end{figure}

\subsubsection{Control - Actor critic}
\label{app_sec:actor_critic}

\begin{algorithm}[ht]
\caption{Preferential TD: Actor-Critic}
\begin{algorithmic}[1]
    \label{PTD_AC}
    \STATE Input: $\gamma$,$\param$, $\fvec$
    \STATE Initialize: $\w_{act}=0, \w_{cri}=0$
    \STATE Output: $\w_{act}, \w_{cri}$
    \FOR{all episodes}
        \FOR{each step in an episode}
        \STATE Compute one-step TD error $\delta_t$
        \STATE Update $\w_{cri}$ using PTD traces
        \ENDFOR
    \STATE Compute forward view returns $G_t^{\param}$ for each state
    \STATE Update $\w_{act}$ using $\nabla_{\w_{act}} \log{\pol(a | s)} \param(s) (G_t^{\param} - \w_{cri}^T \fvec(s))$ for every state
    \ENDFOR
\end{algorithmic}
\label{algo_ptd_AC}
\end{algorithm}

\begin{figure}[ht]
    \centering
    \begin{subfigure}{0.3\textwidth}
        \centering
        \includegraphics[width=0.9\textwidth, height=0.73\textwidth]{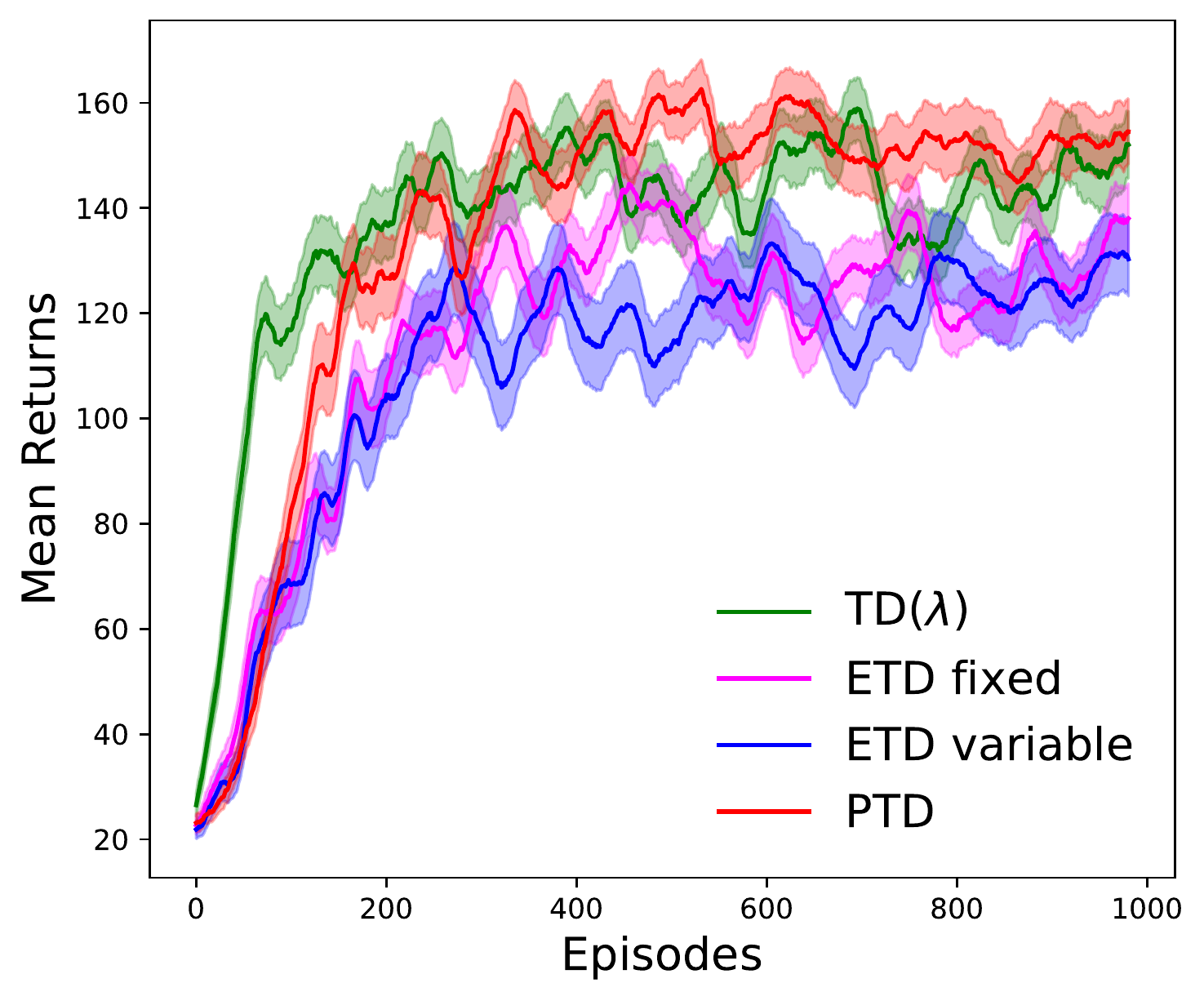}
        \caption{Rewards plot for condition 1}
        \label{app_fig:cartpole_c1_rew}
    \end{subfigure}
    \hfill
    \begin{subfigure}{0.3\textwidth}
        \centering
        \includegraphics[width=0.9\textwidth, height=0.73\textwidth]{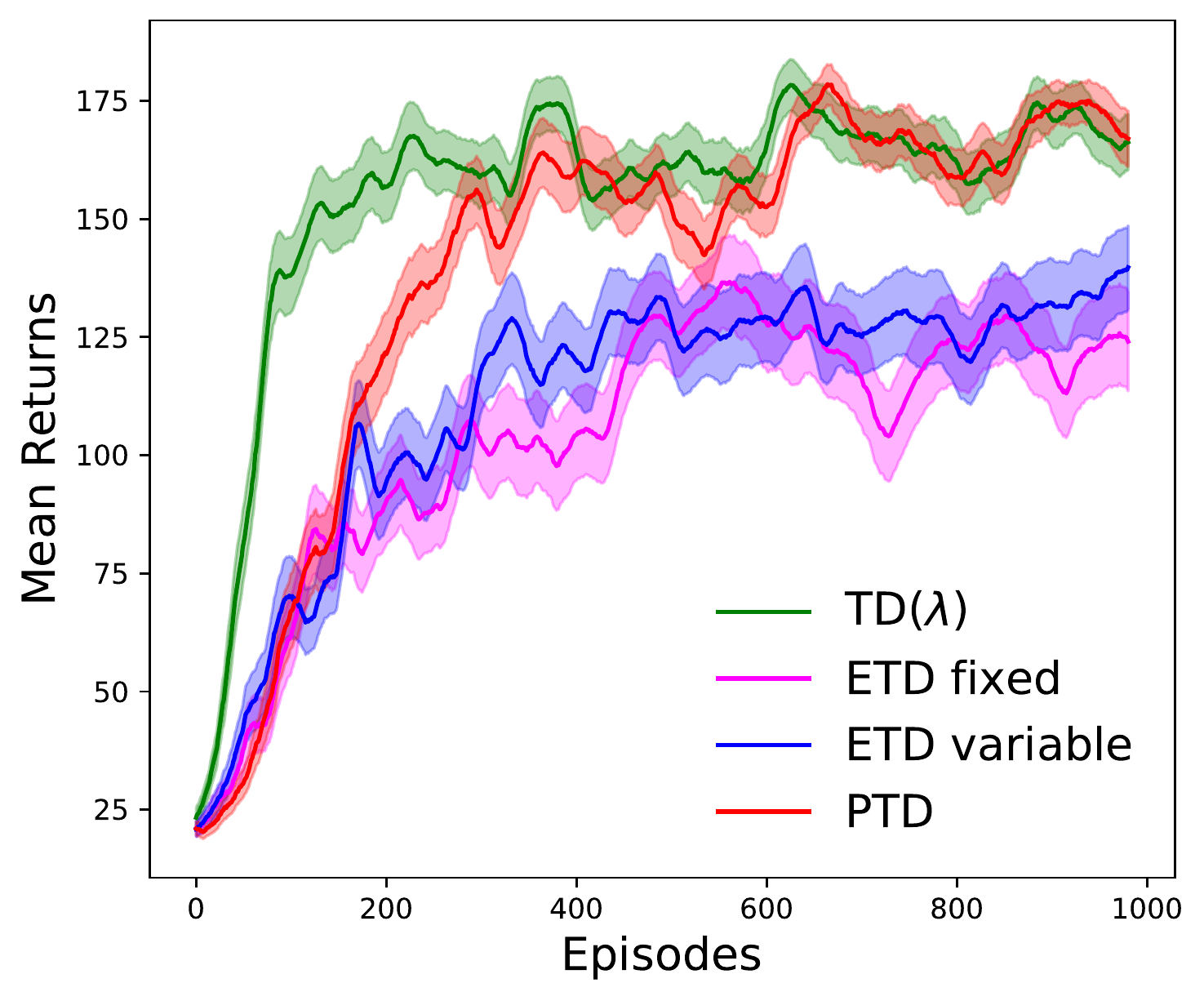}
        \caption{Rewards plot for condition 2}
        \label{app_fig:cartpole_c2_rew}
    \end{subfigure}
    \hfill
    \begin{subfigure}{0.3\textwidth}
        \centering
        \includegraphics[width=0.9\textwidth, height=0.73\textwidth]{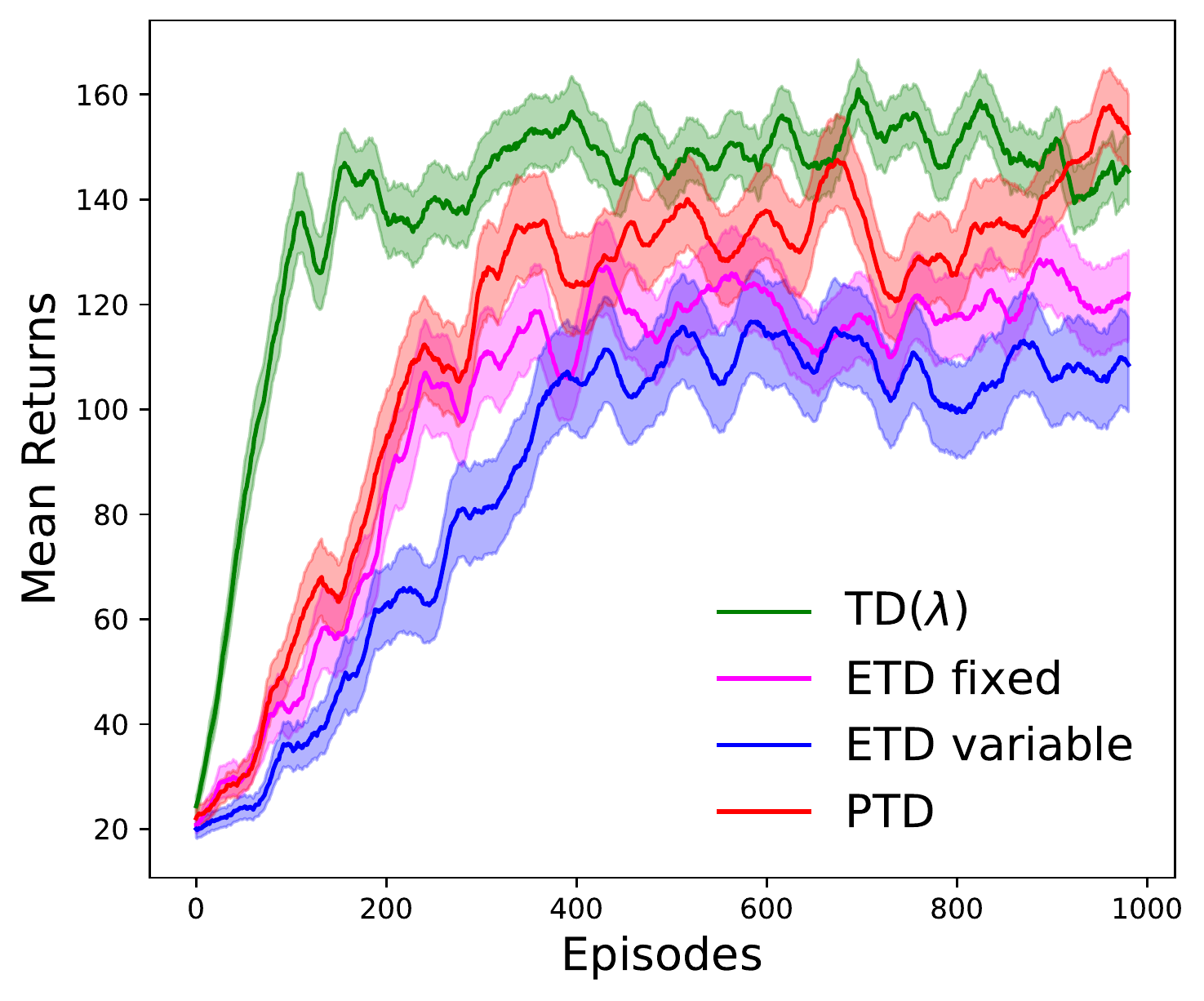}
        \caption{Rewards plot for condition 3}
        \label{app_fig:cartpole_c3_rew}
    \end{subfigure}
    \vfill
    \begin{subfigure}{0.3\textwidth}
        \centering
        \includegraphics[width=0.9\textwidth, height=0.73\textwidth]{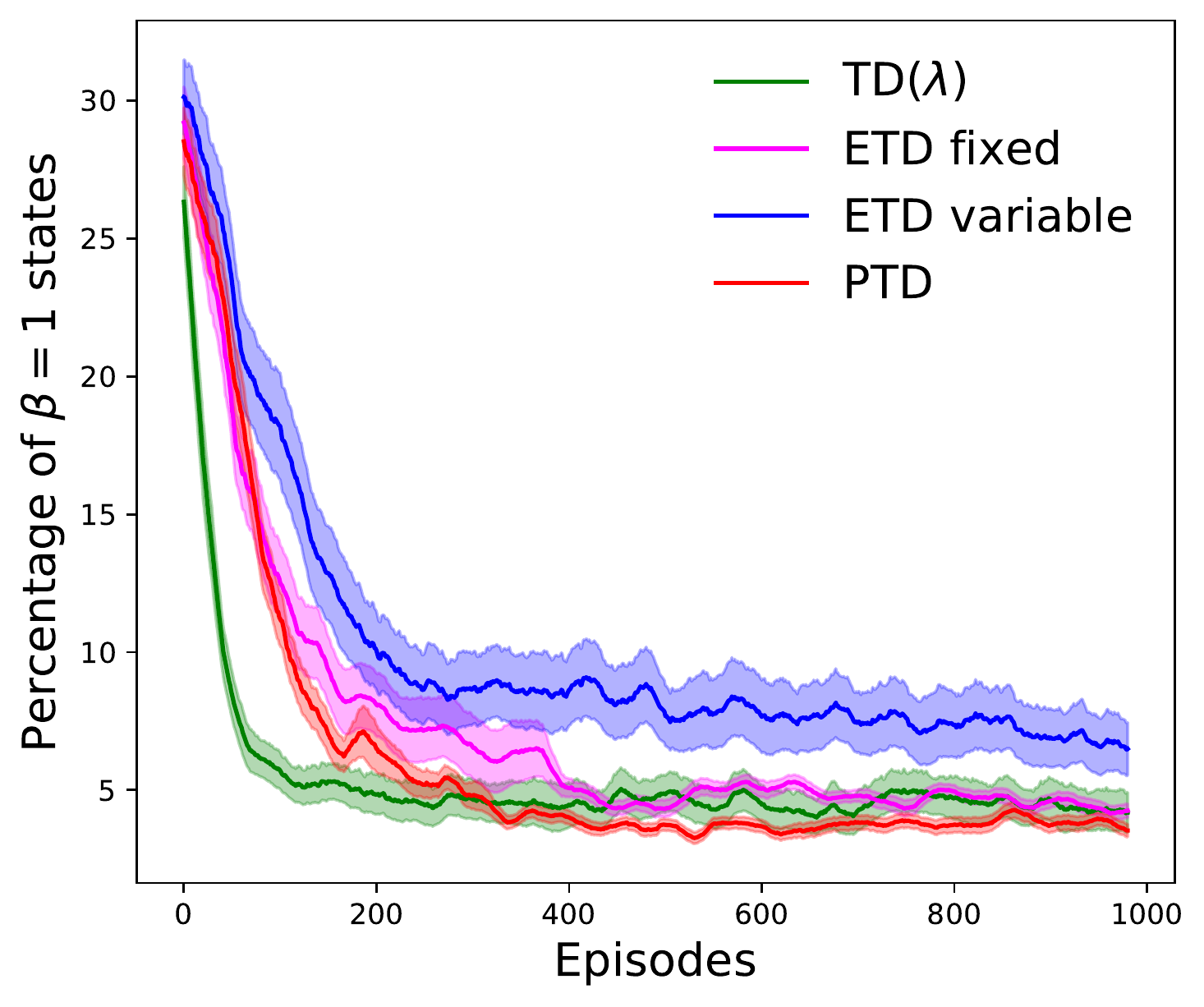}
        \caption{Beta plot for condition 1}
        \label{app_fig:cartpole_c1_beta}
    \end{subfigure}
    \hfill
    \begin{subfigure}{0.3\textwidth}
        \centering
        \includegraphics[width=0.9\textwidth, height=0.73\textwidth]{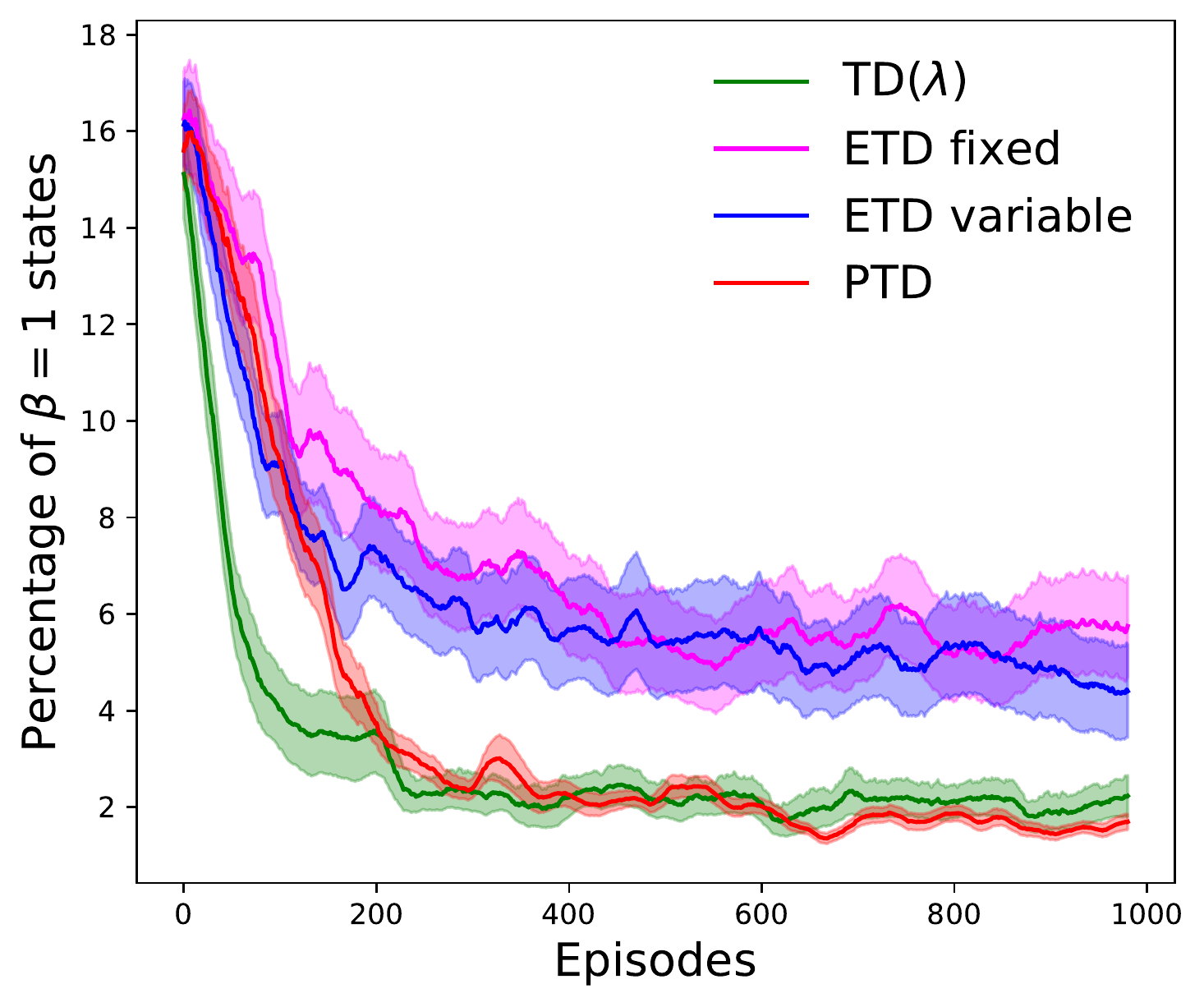}
        \caption{Beta plot for condition 2}
        \label{app_fig:cartpole_c2_beta}
    \end{subfigure}
    \hfill
    \begin{subfigure}{0.3\textwidth}
        \centering
        \includegraphics[width=0.9\textwidth, height=0.73\textwidth]{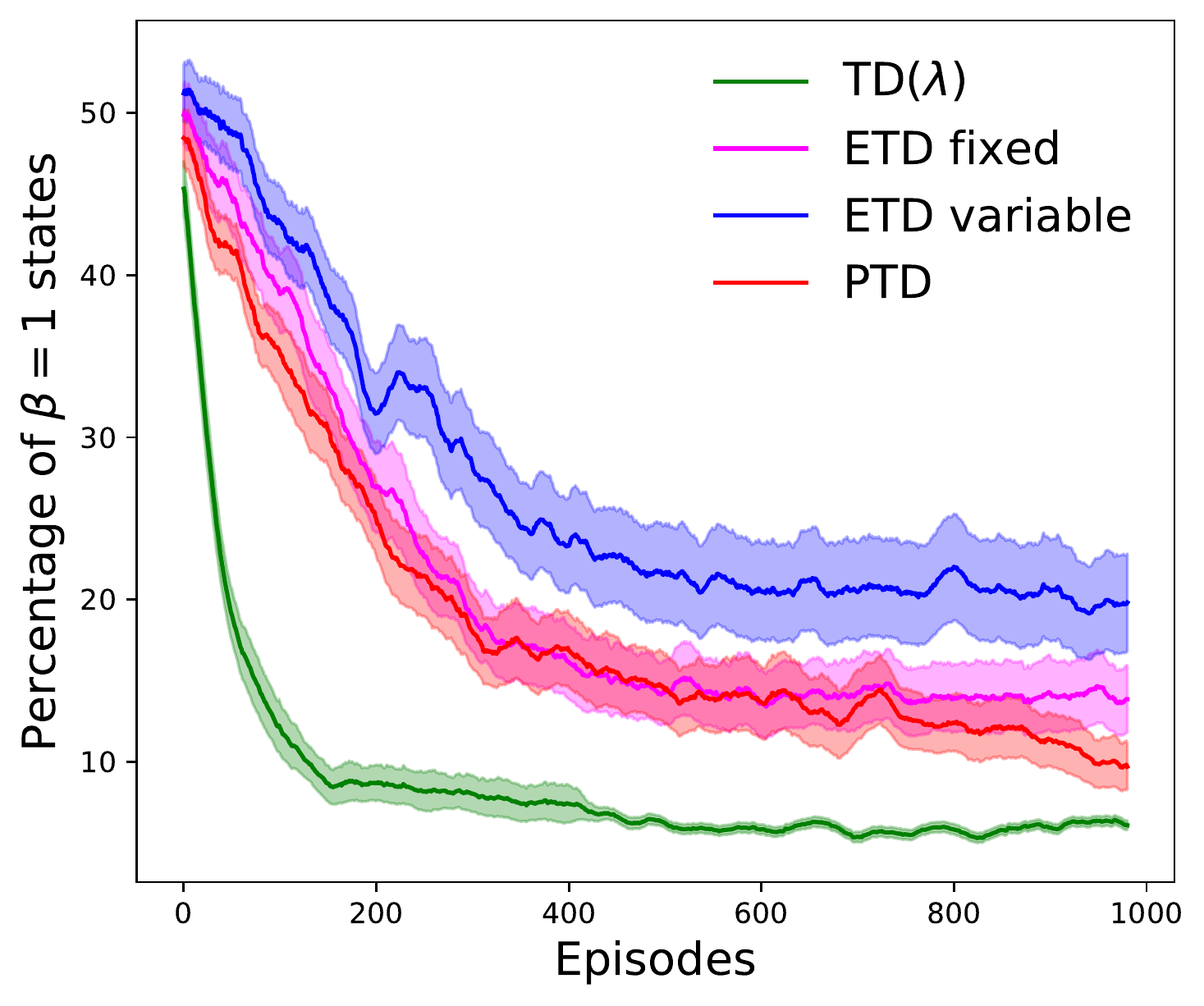}
        \caption{Beta plot for condition 3}
        \label{app_fig:cartpole_c3_beta}
    \end{subfigure}
    \caption{(\textit{Top}) Return is plotted against episodes. (\textit{Bottom}) Percentage of $\param=1$ (or $\lambda=0$) states plotted against episodes. Different columns correspond to different thresholds used to set $\param=1$ (or $\lambda=0$). The plots are smoothed using a moving average window of 20 episodes. The shaded region represents a confidence interval of 50\% on 25 seeds.}
    \label{app_fig:True_AC_cartpole}
\end{figure}
\vspace{-10pt}

\textbf{Task description:} We consider the cartpole task from the openAI gym \citep{brockman2016openai}. The goal of the task is to learn a policy to balance the pole using the algorithm described in \ref{algo_ptd_AC}. We provide this experiment to demonstrate that Preferential TD works (1) when the environment is fully observable, and (2) in control setting.

\textbf{Setup:} We consider actor-critic algorithm for learning the optimal policy. The actor is a single layered neural network with 4 hidden units. The critic is a linear function approximator. The actor is updated at the end of each episode since it is a non-linear function approximator. We use forward view returns to compute the error for updating the actor. The critic is updated at each timestep using eligibility traces. We consider a non-linear approximator for actor as we were not able solve the task using a linear actor on the raw input. We consider the same 4 algorithms for experimentation. For each algorithm we picked the learning rate for the actor from \{0.1, 5e-2, 1e-2, 5e-3, 1e-3\} and for the critic from \{1e-3, 5e-4, 1e-4, 5e-5, 1e-5, 5e-6, 1e-6\}. The best learning rate was decided based on the average returns achieved per episode across 25 seeds and $500$ episodes. The interest for both ETD-variable and ETD-fixed was set to $0.5$. $\param=1$ (or $\lambda=0$) for a state if the difference between the current cart position (or cart velocity) and the cart position (or cart velocity) of the previous $\param=1$ (or $\lambda=0$) state is greater than a threshold. A similar condition was also checked for pole angle with a different threshold. The thresholds for all 3 conditions are provided in table \ref{app_table:cartpole_thresholds}. We experimented with three combinations of threshold values to set $\param$ value to 1. All the states in between two $\param=1$ ($\lambda=0$) states were set to a constant value of $0.1$ ($\lambda=0.9$). We also experimented with $\{0, 0.5\}$ as intermediate values but their performance was relatively poor. All the results are averaged across 25 different seeds. We set the discount factor to $0.99$.

\begin{table}[H]
\caption{Thresholds to determine $\param=1$ state.}
\vspace{10pt}
\centering
\begin{tabular}{ | c | c | c | c | } 
\hline
Condition & Cart position & Cart velocity & Pole angle \\ 
\hline
1 & 0.5 & 0.5 & 0.05 \\ 
2 & 1.0 & 1.0 & 0.1 \\ 
3 & 0.3 & 0.3 & 0.03 \\ 
\hline
\end{tabular}
\label{app_table:cartpole_thresholds}
\end{table}

\textbf{Observations:} The environment is fully observable and therefore we expect to observe the best performance when critic and actor are updated for all the states. TD($\lambda$) updates the values and policy for all the states and therefore it achieves the best performance as seen in Figure \ref{app_fig:True_AC_cartpole}. ETD-fixed and ETD-variable achieves a similar performance. The performances are substantially poorer compared to the other methods. We suspect that it is because of high variance in the updates resulting from trajectory dependent emphasis. This observation is also validated from the beta plots (Figure \ref{app_fig:True_AC_cartpole}), where we observe significantly higher $\param=1$ states. The performances drop further for condition 3, where the percentage of $\param=1$ states are much higher compared to the other conditions. Besides, hyperparameter tuning was found to be challenging as these algorithms were sensitive to them. Preferential TD's performance was as good as TD($\lambda$). However, PTD updates and bootstraps from a small percentage (roughly 5\%) of states. PTD learning is not sharp at the beginning when the estimates are poor. But it catches up quickly and the overall performance is on par with TD($\lambda$). We also observed a much lower variance across seeds in PTD compared to the other algorithms. This experiment validates our hypothesis that a similar performance can be achieved by bootstrapping and updating from a small set of states as opposed to all. Also, this experiment provides preliminary evidence for PTD's usage in fully observable setting and in settings beyond policy evaluation. Although further analysis, both theoretical and empirical, is required to make meaningful conclusions.

\end{document}